\documentclass{article}

\usepackage{microtype}
\usepackage{graphicx}
\usepackage{subcaption}
\usepackage{booktabs} 

\usepackage{svg}
\usepackage{color}
\usepackage{float}   
\usepackage{wrapfig}
\usepackage{hyperref}
\usepackage{adjustbox}

\usepackage[preprint]{icml2026}

\usepackage{amsmath}
\usepackage{amssymb}
\usepackage{mathtools}
\usepackage{amsthm}

\usepackage{booktabs}
\usepackage{longtable}
\usepackage{multirow}

\usepackage[capitalize,noabbrev]{cleveref}

\theoremstyle{plain}
\newtheorem{theorem}{Theorem}[section]
\newtheorem{proposition}[theorem]{Proposition}
\newtheorem{lemma}[theorem]{Lemma}
\newtheorem{corollary}[theorem]{Corollary}
\theoremstyle{definition}

\newtheorem{assumption}[theorem]{Assumption}
\theoremstyle{remark}

\usepackage[textsize=tiny]{todonotes}

\icmltitlerunning{{T}ime Series {I}n-Context {C}lassification {F}oundation {M}odel}

\begin{document}

\twocolumn[
  \icmltitle{Rethinking Zero-Shot Time Series Classification: \\From Task-specific Classifiers to In-Context Inference}



  \icmlsetsymbol{equal}{*}

  \begin{icmlauthorlist}
    \icmlauthor{Juntao Fang}{equal,yyy}
    \icmlauthor{Shifeng Xie}{equal,aaa1,comp}
    \icmlauthor{Shengbin Nie}{yyy}
    \icmlauthor{Yuhui Ling}{yyy}
    \icmlauthor{Yuming Liu}{yyy}
    \icmlauthor{Zijian Li}{sch}
    \icmlauthor{Keli Zhang}{aaa1}
    \icmlauthor{Lujia Pan}{aaa2}
    \icmlauthor{Themis Palpanas}{comp}
    \icmlauthor{Ruichu Cai}{yyy}
  \end{icmlauthorlist}

  \icmlaffiliation{yyy}{Guangdong University of Technology, Guangzhou, China}
  \icmlaffiliation{comp}{Paris Descartes University, Paris, France}
  \icmlaffiliation{sch}{Mohamed bin Zayed University of Artificial Intelligence, Abu Dhabi, United Arab Emirates}
  \icmlaffiliation{aaa1}{Huawei Noah’s Ark Lab, Paris, France}
  \icmlaffiliation{aaa2}{Huawei Noah’s Ark Lab, Shenzhen, China}
  \icmlcorrespondingauthor{Ruichu Cai}{cairuichu@gmail.com}

  \icmlkeywords{XXXX, XXXXX}

  \vskip 0.3in
]



\printAffiliationsAndNotice{\icmlEqualContribution}

\begin{abstract}
The zero-shot evaluation of time series foundation models (TSFMs) for classification typically uses a frozen encoder followed by a task-specific classifier. However, this practice violates the training-free premise of zero-shot deployment and introduces evaluation bias due to classifier-dependent training choices. To address this issue, we propose TIC-FM, an in-context learning framework that treats the labeled training set as context and predicts labels for all test instances in a single forward pass, without parameter updates. TIC-FM pairs a time series encoder and a lightweight projection adapter with a split-masked latent memory Transformer. We further provide theoretical justification that in-context inference can subsume trained classifiers and can emulate gradient-based classifier training within a single forward pass. Experiments on 128 UCR datasets show strong accuracy, with consistent gains in the extreme low-label situation, highlighting training-free transfer for time series classification.The source code is publicly available at https://github.com/fangjuntao/TIC-FM.

\end{abstract}


\section{Introduction}
Time series classification is a core task across applications such as human movement analysis, clinical monitoring and diagnosis, digital health, and energy systems, where accurate recognition of temporal patterns directly supports decision-making \citep{TSClassificationSurvey}. Recent progress on time series foundation models (TSFMs) makes it increasingly feasible to deploy a single pretrained backbone across domains \citep{TSFMsurvey}. At the same time, in privacy-sensitive or labor-intensive settings, like healthcare, collecting and annotating labeled time series at scale is often costly or infeasible, which amplifies the need for transfer without extensive task-specific training \citep{gao2025data}. Consequently, strong zero-shot and few-shot generalization is not merely desirable but central to realizing practical for classification with TSFMs.


Despite these aspirations, the evaluation practice for TSFMs remains largely limited to a frozen encoder and task-specific classifier practice. In this setup, the backbone is kept frozen, and a task-specific classifier (e.g., an SVM or an MLP) is trained on top of the extracted embeddings to adapt to downstream datasets. While widely adopted, we argue that this protocol suffers from fundamental methodological shortcomings. First, training task-specific classifiers requires explicit parameter optimization and access to a training set for hyperparameter tuning. This practice conflicts with the zero-shot premise of foundation models, since the model still requires a supervised training step before deployment. Second, the protocol induces substantial evaluation bias and sensitivity. As our empirical analysis shows (see ~\cref{StandardUCR}), the effectiveness of a frozen backbone depends heavily on the choice of task-specific classifier (e.g., SVM vs. MLP). Such dependence inevitably conflates the intrinsic quality of the learned representations with the expressivity of the classifier, thereby obstructing a fair pipeline of the foundation model itself. Finally, training a task-specific classifier is less reliable in the extremely low-shot regime. Optimizing a high-dimensional classifier generally demands a nontrivial sample size for stable convergence; consequently, under data scarcity (e.g., $\leq$ 10), conventional classifiers are more likely to overfit or fail to converge reliably.

To overcome these limitations, we propose a paradigm shift from training task-specific classifiers to  In-Context Learning (ICL). Specifically, we introduce the  \textbf{T}ime Series \textbf{I}n-Context \textbf{C}lassification \textbf{F}oundation \textbf{M}odel (TIC-FM), a framework that treats the entire labeled support set as a contextual prompt without explicit weight updates. By jointly attending to context and query instances, TIC-FM infers task-specific decision boundaries within a single forward pass. The method performs inference without any weight updates, thereby eliminating the need for gradient optimization, classifier selection, and test-time hyperparameter tuning.  Architecturally, TIC-FM combines a time series encoder with a lightweight projection adapter and a split-masked in-context Transformer classifier that uses latent memory to summarize long contexts and injects labels only into support tokens for leakage free parallel inference.

Our main contributions are summarized as follows:
\begin{itemize}
\item  We revisit the prevailing practice of evaluating frozen backbones via task-specific classifiers and instead propose an ICL pipeline. By avoiding classifier selection and its associated hyperparameters, the resulting evaluation is both fairer and more robust.
\item We propose TIC-FM, a specialized ICL architecture designed for time series classification inference. Across 128 UCR datasets, TIC-FM consistently outperforms existing time series foundation models.
\item Building on a well developed proof framework, we formalize how our ICL pipeline emulates gradient descent behavior without explicit parameter updates, offering a theoretical lens on its robust generalization.
\end{itemize}



\section{Related Work}
\textbf{Time series foundation model (TSFM)} leverage large-scale pretraining to learn transferable ability that can be deployed in a zero-shot manner on new domains, mirroring the foundation model paradigm in large language model and vision \citep{TSFMsurvey}. Recent progress has been especially striking in forecasting, where large pretrained models support standardized deployment modes spanning zero-shot evaluation, few-shot adaptation, and task-specific fine-tuning \citep{ansari2024chronos,ansari2025chronos2,timeFM,toto,moirai,auer:25tirex,tempopfn,lagllama}.  In contrast, zero-shot time series classification is still commonly operationalized as frozen feature extraction plus a trained downstream classifier, which is not train-free when moving to a new labeled dataset \citep{mantis,nutime,timesbert,cauker,auer:25tirexclassification,units,gpt2TS}.  Among classification oriented TSFMs, models such as MOMENT \citep{moment} popularize the embedding-classifier paradigm and have become widely used backbones, but a truly train-free zero-shot classification foundation model remains largely underexplored.

\textbf{In-context learning (ICL)}, popularized by large language models, refers to task adaptation at inference time by conditioning on a small set of input–output exemplars, without any parameter updates  \citep{ICLSurvey}.  Prevailing mechanistic accounts of ICL frame it either as implicit Bayesian inference over latent task structure \citep{implicitBayesian,Bayesian2,Bayesian3} or as optimization-like dynamics implemented within the transformer forward pass  \citep{ICLPCA,ahnICLGD,transformerslearnincontextgradient,ICLGDOptimal,xie2025incontextlearning}.  In contrast to the rich LLM literature, the in-context capability of TSFMs has been far less explored, and in forecasting the notion of “context” is often conflated with the historical lookback window rather than exemplar-based conditioning \citep{timesnet,auer:25tirex,lu2025incontext,faw2025incontext}. Meanwhile, time series classification naturally matches the ICL paradigm: context examples can be formed as sequence–label pairs, enabling train-free prediction of query labels conditioned on a labeled context set in a single forward pass.  This sequence–label formulation also mirrors recent  ICL approaches for tabular foundation models \citep{tabpfn,tabicl}, reinforcing an emerging unification trend across foundation models for structured data \citep{pfn,tabpfnTS,cauker}. Concurrently, \citet{tokic2025tsfm} and \citet{yeh2025tict} have also started to investigate in-context learning for time series classification. \citet{tokic2025tsfm} focus on an application-driven setting, while \citet{yeh2025tict} emphasizes a synthetic data pre-training method. In contrast, we argue that the dominant “frozen encoder + trained classifier” evaluation protocol is not truly zero-shot, and we instead formalize train-free inference for classification TSFMs.

\section{Preliminary}
\label{sec:setting_motivation}

\textbf{Time series classification tasks.}
A time series classification task is a supervised problem defined by a labeled dataset
\[
\mathcal{D}=\{(x_i,y_i)\}_{i=1}^{n},\qquad x_i\in\mathbb{R}^{d\times t},\; y_i\in\{1,\dots,K\},
\]
where $n$ denotes the total number of labeled instances, $d$ is the number of channels per instance ($d=1$ for univariate series), and $t$ is the sequence length (number of time steps). The dataset is split into a training set $\mathcal{D}_{\mathrm{tr}}$ and a test set $\mathcal{D}_{\mathrm{te}}$.

\textbf{Standard zero-shot pipeline for TSFM classification.}
Let $F_\psi$ denote a pretrained TSFM used as a frozen feature encoder,
\[
F_\psi:\mathbb{R}^{d\times t}\rightarrow \mathbb{R}^{q},\qquad z_i = F_\psi(x_i).
\]
The evaluation protocol \citep{moment, mantis, nutime} for zero-shot TSFM-based classification is feature extraction:
given a dataset $\mathcal{D}_\tau$ with training and test splits
$\mathcal{D}_{\mathrm{tr}}=\{(x_i,y_i)\}_{i=1}^{N_{\mathrm{tr}}}$ and
$\mathcal{D}_{\mathrm{te}}=\{(x_j,\cdot)\}_{j=1}^{N_{\mathrm{te}}}$,
we compute embeddings $z_i = F_\psi(x_i)\in\mathbb{R}^{q}$ for all training samples and train a task-specific classifier $h_\tau:\mathbb{R}^{q}\rightarrow \mathbb{R}^{K}$ on the embedded training set $\{(z_i,y_i)\}_{i=1}^{N_{\mathrm{tr}}}$, and report predictions on test embeddings:
\[
\hat{y}(x)=h_\tau(F_\psi(x)),\qquad (x,\cdot)\in\mathcal{D}_{\mathrm{te}}.
\]

This pipeline is widely adopted in the literature on TSFMs for classification: first, the encoder $F$ is frozen and used to map a fixed length time series to a representation; then, a task-specific classifier $h$ is trained on the resulting embeddings (e.g., a random forest (RF) or a support vector machine (SVM)). We refer to this evaluation protocol as the {``freeze backbone and task-specific classifier''} pipeline throughout the paper.

For clarity, the few-shot pipeline commonly used with TSFMs, both for forecasting and classification, fine-tunes the model on a small labeled set (few-shot set) before testing \citep{moment,lagllama}. This adaptation setting is orthogonal to our zero-shot pipeline.

\textbf{In-context learning for classification.}
In contrast, ICL treats the labeled training set as context and predicts test labels in a single forward pass, without any parameter updates. Formally, it learns a model that  outputs $~p(y_{\mathrm{te}} \mid x_{\mathrm{te}}, \mathcal{D}_{\mathrm{tr}}$), namely the class probabilities for a test sample conditioned on the training set, and it can generate predictions for the test set in the same forward pass. 

In this work, we focus on a stricter train-free setting. For each dataset (i.e., task), we treat the labeled training split as a {context set} and the test split as a {query set}, and aim to predict query labels {without fitting any task-specific parameters}, i.e., using only forward computation. We realize this goal via in-context inference, detailed in Section \ref{TIC-FM}.

\section{TIC-FM Methodology}
\label{TIC-FM}

\subsection{Embedding-to-ICL Framework}
\label{Framework}
Following the train-free ICL setting in Section~\ref{sec:setting_motivation}, TIC-FM performs a single forward pass over the concatenated context and query sets, producing predictions for all query instances in parallel. A schematic overview of the architecture is presented in Figure~\ref{fig:ticfm_overview}.  Specifically, TIC-FM comprises three components with a clear information-flow constraint. First, a time series encoder $F_{\psi}$ summarizes the within-sequence temporal structure into an instance embedding that supports cross-instance comparison. Next, a projection adapter $g_{\phi}$ maps these embeddings into the token space expected by the in-context Transformer. Finally, an in-context classifier $G_{\theta}$ performs inference by attending to the labeled context tokens and the unlabeled query tokens while enforcing a strict separation between the two sets, thereby enabling parallel prediction for all query samples.


\subsection{Time Series Feature Encoder}
\label{TimeSerisesEncoder}

We use a time series encoder $F_{\psi}$ to map each input series $x\in\mathbb{R}^{T}$ to an embedding
$z = F_{\psi}(x)\in\mathbb{R}^{q}$.
The encoder first employs a token generator unit to transform the raw series into a sequence of patch tokens that capture local temporal dynamics and patch-level statistics, and then applies a ViT to aggregate information across the token sequence. The final embedding $z$ is obtained from a designated pooling token (i.e., a classification token), which attends to all patch tokens to summarize their information and thus yields a semantically rich embedding of $x$.

\textbf{Token Generator Unit.} Drawing inspiration from Mantis~\citep{mantis}, the token generator constructs a sequence of patch embeddings
$U \in \mathbb{R}^{P \times q}$ that integrates both local dynamics and patch-level statistical properties. We combine local dynamics with patch-level statistics to improve invariance to shifts and to preserve discriminative cues when the labeled context is limited. Specifically, we partition the input time series $x \in \mathbb{R}^{T}$ into $P$ non-overlapping patches of length
$w = T/P$ (assuming $T$ is divisible by $P$). Let $x^{(p)} \in \mathbb{R}^{w}$ denote the $p$-th patch for
$p=1,\ldots,P$. We first compute the mean and standard deviation for each patch to capture patch-level statistical properties:
\begin{equation*}
\mu_p = \mathrm{mean}(x^{(p)}) \in \mathbb{R}, \qquad
\sigma_p = \mathrm{std}(x^{(p)}) \in \mathbb{R}.
\end{equation*}
These statistics are encoded via a learnable statistic encoder $\phi_{\mathrm{stat}}(\cdot)$:
\begin{equation*}
s_p = \phi_{\mathrm{stat}}([\mu_p;\sigma_p]) \in \mathbb{R}^{d_s},
\end{equation*}
where $[\cdot;\cdot]$ denotes concatenation and $d_s$ is the statistic-embedding dimension. To extract short-term temporal patterns, we construct two complementary views of the input: the raw series $x$
and its first-order difference $\Delta x$ (zero-padded to length $T$). We employ a shared local feature extractor
$\phi_{\mathrm{loc}}(\cdot)$ to obtain patch-aligned dynamic representations:
\begin{equation*}
\begin{aligned}
V^{\mathrm{raw}}  &= \phi_{\mathrm{loc}}(x) \in \mathbb{R}^{P \times d_\ell}, \\
V^{\mathrm{diff}} &= \phi_{\mathrm{loc}}(\Delta x) \in \mathbb{R}^{P \times d_\ell},
\end{aligned}
\end{equation*}
where $d_\ell$ denotes the dynamic feature dimension. In our implementation, $\phi_{\mathrm{loc}}$ is instantiated
as a stack of convolutional layers followed by patch-wise pooling, ensuring alignment with the patch segmentation. Finally, we fuse the dynamic and statistical features via concatenation and a learnable projection
$\phi_{\mathrm{tok}}$:
\begin{equation*}
\begin{aligned}
u_p &= \phi_{\mathrm{tok}}\!\left(\left[\,V^{\mathrm{diff}}_p;\,V^{\mathrm{raw}}_p;\,s_p\,\right]\right)
\in \mathbb{R}^{q}, \\
U &= [u_1,\ldots,u_P] \in \mathbb{R}^{P \times q},
\end{aligned}
\end{equation*}
where $V^{\mathrm{diff}}_p$ and $V^{\mathrm{raw}}_p$ denote the $p$-th rows of their respective matrices.
In this unit, patch tokens are constructed as the basic evidence units for context-based matching. We combine local dynamics  with patch-level statistics  to improve invariance to shifts and to preserve discriminative cues under limited labeled context.

\textbf{ViT Unit.} We adopt a ViT Unit \citep{VIT,mantis}  to summarize the patch tokens because self-attention implements content-adaptive pooling, allowing the model to emphasize informative temporal segments that are most relevant for classification. This yields a compact and discriminative instance representation that is well suited for downstream in-context classification. Given the patch embedding sequence $U=[u_1,\ldots,u_P]\in\mathbb{R}^{P\times q}$ produced by the token generator above,  the ViT unit is used to aggregate information across patches and obtain a global representation.
Specifically, we introduce a learnable classification token $u_{\mathrm{cls}}\in\mathbb{R}^{q}$ and prepend it to
the patch tokens:
\begin{equation*}
\hat{U}^{(0)} = [u_{\mathrm{cls}}; U]\in\mathbb{R}^{(P+1)\times q}.
\end{equation*}
We then add positional encodings $E_{\mathrm{pos}}\in\mathbb{R}^{(P+1)\times q}$ to inject patch order information,
\(
\tilde{U}^{(0)} = \hat{U}^{(0)} + E_{\mathrm{pos}},
\)
and feed the resulting sequence into an $L$-layer Transformer:
\begin{equation*}
\tilde{U}^{(L)} = \mathrm{Transformer}_{\psi}\!\left(\tilde{U}^{(0)}\right)\in\mathbb{R}^{(P+1)\times q}.
\end{equation*}
Finally, we take the output corresponding to the classification token as the instance embedding:
\(
z = \tilde{U}^{(L)}_{0}\in\mathbb{R}^{q}.
\)
In our implementation, the Transformer uses multi-head self-attention with $h$ heads and an MLP feed-forward block at
each layer, matching the standard ViT design.

\begin{figure}[t]
  \centering
  \includegraphics[width=0.98\linewidth]{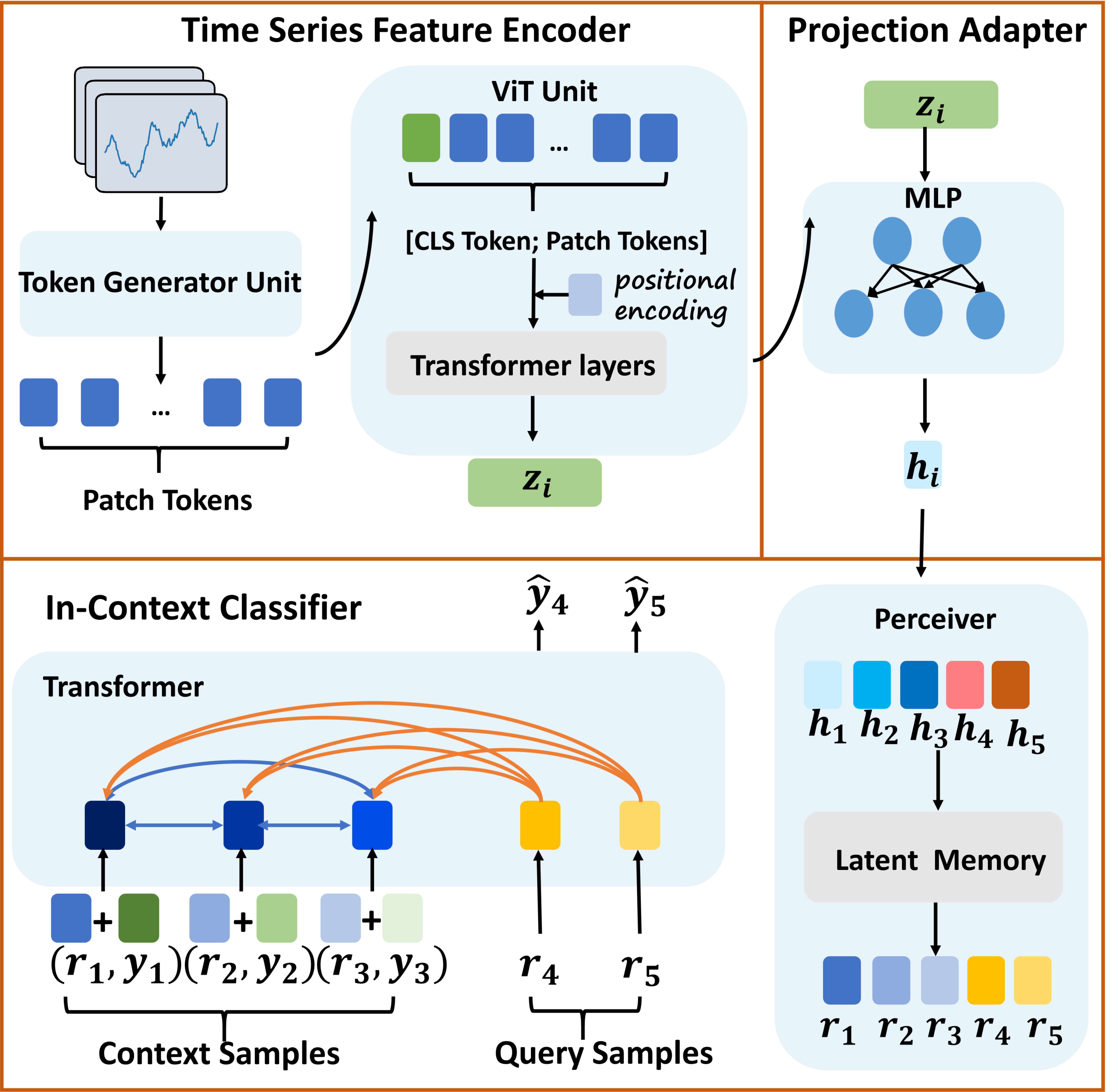}
  \caption{An overview of TIC-FM architecture.  Each time series is first encoded by a ViT-based feature encoder into an instance embedding, and then mapped by a lightweight projection adapter to the token space of the in-context classifier. The classifier processes all context and query samples jointly: it consolidates long contexts via perceiver latent memory, injects label embeddings only into the context slice, and performs split-masked Transformer reasoning.}
  \label{fig:ticfm_overview}
\end{figure}

\subsection{Projection Adapter}
\label{sec:proj}
The encoder embedding space $\mathbb{R}^{q}$ is not necessarily aligned with the token
space expected by the in-context classifier. Since in-context inference is driven by attention weights and token-wise interactions, mismatched feature statistics can degrade conditioning on labeled context. We therefore learn a lightweight projection
adapter $g_{\phi}$ to map instance embeddings into an ICL-compatible representation space. We implement $g_{\phi}$ as an MLP with Layer Normalization \citep{LN}:
\begin{equation*}
\begin{aligned}
g_{\phi}(z) = W_2\,\sigma\!\bigl(W_1\,\mathrm{LN}(z)\bigr), 
W_1 \in \mathbb{R}^{q\times d_h}, 
W_2 \in \mathbb{R}^{d_h\times d},
\end{aligned}
\end{equation*}

where $\sigma$ is GELU \cite{GELU}. The adapter is applied independently to
each instance embedding, enabling efficient batch processing of dataset sequences.
In our implementation, $d$ matches the internal model dimension of the ICL Transformer.

\subsection{In-Context Classifier}
\label{sec:icl}

We denote the projected feature embeddings from the upstream encoder as $H_{\mathrm{tr}} \in \mathbb{R}^{N_{\mathrm{tr}} \times d}$ and $H_{\mathrm{te}} \in \mathbb{R}^{N_{\mathrm{te}} \times d}$.
The In-Context Classifier, denoted as $G_\theta$, processes the concatenated sequence $H^{(0)} = [H_{\mathrm{tr}}; H_{\mathrm{te}}] \in \mathbb{R}^{N \times d}$, where $N = N_{\mathrm{tr}} + N_{\mathrm{te}}$. The computation consists of three stages: latent memory consolidation,
supervision injection, and split-masked in-context reasoning.

\textbf{Latent Context Consolidation.}
To summarize long contexts efficiently, we incorporate a Perceiver-style latent
memory mechanism ~\cite{orion} before injecting labels.
The mechanism maintains learnable latent queries $L \in \mathbb{R}^{M \times d}$ with $M \ll N_{\mathrm{tr}}$,
and applies a stack of cross-attention and feed-forward blocks.
Concretely, it first writes information from the context tokens into the latents and then reads the updated latents back to refine all tokens:
\begin{equation*}
\tilde L = \mathrm{Write}\!\left(L,\, H_{\mathrm{tr}}\right),
\qquad
H^{(0)} \leftarrow \mathrm{Read}\!\left(H^{(0)},\, \tilde L\right).
\end{equation*}
Crucially, the write operation only accesses context features without labels, and the subsequent read refines both context and query representations via the latent summary, thereby remaining label-leak safe.

\textbf{Task-Aware Label Injection.}
We encode class labels using a one-hot projection
$E_y:\{0,\ldots,C_{\max}-1\}\rightarrow\mathbb{R}^{d}$ implemented by a linear layer over one-hot inputs.
Labels are injected additively {only} into the context slice:
\begin{equation*}
H^{(0)}_{1:N_{\mathrm{tr}}} \leftarrow H^{(0)}_{1:N_{\mathrm{tr}}} + E_y\!\left(y_{\mathrm{tr}}\right),
\qquad
H^{(0)}_{N_{\mathrm{tr}}+1:N} \ \text{unchanged},
\end{equation*}
so that ground-truth labels act as in-context prompts while query tokens remain unlabeled.

\textbf{Split-masked In-Context Reasoning.}
The label-injected sequence is processed by an $L$-layer Transformer encoder. To enforce a strict separation between context and queries, we apply a split attention mask $\mathcal{M}(N_{\mathrm{tr}})$ defined by the boundary between the context and the query tokens. Under this mask, context tokens attend only within the context set, and each query token attends only to the context tokens. This masking prevents information leakage between query tokens. The encoder outputs are then normalized and decoded
into class logits:
\begin{equation*}
\begin{aligned}
H^{(L)} &= \mathrm{Transformer}\!\left(H^{(0)};\, \mathcal{M}(N_{\mathrm{tr}})\right), \\
\ell &= D_\omega\!\big(\mathrm{LN}(H^{(L)})\big),
\end{aligned}
\end{equation*}
where $D_\omega$ is a token-wise two-layer MLP producing $C_{\max}$ logits. For prediction, the decoder produces per-token logits $\ell \in \mathbb{R}^{N\times C_{\max}}$.
Let $\mathcal{S}_\mathrm{act} \subseteq \{0,\dots,C_{\max}-1\}$ be the index set of the $K$ classes appearing in the context labels
(after task-wise re-indexing). We take the query slice and restrict it to the active classes:
\[
\ell_{\mathrm{te}} := \ell_{N_{\mathrm{tr}}+1:N,\ \mathcal{S}_\mathrm{act}} \in \mathbb{R}^{N_{\mathrm{te}}\times K},\qquad |\mathcal{S}_\mathrm{act}|=K,
\]
and apply a temperature-scaled softmax $\hat{p}(y\mid x)=\mathrm{Softmax}_{\tau}\!\left(\ell_{\mathrm{te}}/\tau\right)$.

\subsection{Pretraining and Inference}For pretraining, we adopt a stage-wise training strategy tailored to each component’s role. We first use Cauker~\cite{cauker} to generate 100K synthetic time series samples and train the feature encoder $F_{\psi}$ for 100 epochs with a contrastive objective. We then pretrain the in-context classifier $G_{\theta}$ on synthetic data drawn from a structural causal model prior~\cite{orion}. Finally, using the UCR training splits, we freeze $F_{\psi}$ and $G_{\theta}$ and train only the projection adapter $g_{\phi}$ for 5 epochs with cross-entropy. No information from UCR test splits is used at any stage. 
Alternatively, $g_{\phi}$ can also be trained purely on SCM data, yielding a fully synthetic pre-trained TIC-FM, further details are provided in Appendix~\ref{appendixPretrain}.
For inference, we apply an ensembling strategy to improve robustness. Given a context set $\mathcal{D}_{\mathrm{tr}}$ and a query set $\mathcal{D}_{\mathrm{te}}$ with $K$ classes, we form
$X_{\mathrm{all}}=[X_{\mathrm{tr}};X_{\mathrm{te}}]$ and evaluate $M$ ensemble members.
Each member is defined by  a cyclic label permutation applied {only} to the context labels, $\pi_{o_m}(y)=(y+o_m)\bmod K$.
We run $G_{\theta}$ to obtain query logits $\ell_m(\cdot)$ under $o_m$, map them back to the original label space via
$\pi_{o_m}^{-1}$, and aggregate across members:
\begin{equation*}
\hat{p}(y\mid x)=\mathrm{Softmax}_{\tau}\!\left(\frac{1}{M}\sum_{m=1}^{M}\pi_{o_m}^{-1}\!\left(\ell_m(x; o_m)\right)\right).
\end{equation*} Additionally, when $K$ exceeds $C_{\max}$, we adopt a hierarchical class-extension strategy~\cite{tabicl} by constructing a balanced class-partition tree that recursively decomposes the $K$-way task into subproblems, each involving at most $C_{\max}$ classes, and combines group and within-group predictions via the law of total probability. Further inference details are provided in Appendix~\ref{app:inference}.

\section{Why In-Context Inference Works}
\label{sec:theory}

\subsection{ICL Contains Trained Classifiers}
\label{sec:icl_contains_head}

\begin{assumption}[Bounded input domain and symmetry]
\label{ass:prompt_domain_sym}
We fix maximum sizes $n_{\mathrm{tr}}\le N$ and $n_{\mathrm{te}}\le M$ and pad context--query input accordingly.
We assume embeddings lie in a compact set and consider the resulting compact padded input domain $\mathcal{X}$.
The induced  score prediction map is permutation-invariant with respect to the order of training examples and equivariant in test examples.

\end{assumption}

\begin{proposition}[ICL subsumes the trained-classifier zero-shot pipeline]
\label{prop:icl_contains_head}
Let $\mathsf{Train}$ be any classifier-training procedure that maps a training set to a classifier $h_\tau:\mathbb{R}^q\to\mathbb{R}^K$.
Let the corresponding  \emph{score} map be
$
F(\mathcal{D}_{\mathrm{tr}},\{z_j^{\mathrm{te}}\}_{j=1}^{n_{\mathrm{te}}})
\triangleq
\{h_\tau(z_j^{\mathrm{te}})\}_{j=1}^{n_{\mathrm{te}}}
\in(\mathbb{R}^K)^{n_{\mathrm{te}}}
$
with $h_\tau=\mathsf{Train}(\mathcal{D}_{\mathrm{tr}})$.
Assume $f$ satisfies Assumption~\ref{ass:prompt_domain_sym} and is continuous on $\mathcal{X}$.

Then for any $\varepsilon>0$, there exists an in-context classifier $G_\theta$ operating on the prompt tokens
$
\{r(z_i^{\mathrm{tr}},y_i^{\mathrm{tr}})\}_{i=1}^{n_{\mathrm{tr}}}
\cup
\{r(z_j^{\mathrm{te}},\bot)\}_{j=1}^{n_{\mathrm{te}}}
$
such that, on $\mathcal{X}$,
\[
\sup\ \max_{1\le j\le n_{\mathrm{te}}}\ \big\|G_\theta(\cdot)_j - h_\tau(z_j^{\mathrm{te}})\big\|_\infty
\le \varepsilon.
\]
\end{proposition}
In particular, this implies the following label-level guarantee under a uniform margin condition.

\textbf{Proof sketch.}
View the trained-classifier pipeline as a continuous map that takes a set of labeled training tokens and a set of test tokens and returns test-time scores. Because the map is invariant to permutations of the training set, we can uniformly approximate each output coordinate by a polynomial that is symmetric in the training tokens. Any symmetric polynomial in the training tokens can be rewritten as a function of finitely many aggregated features of the form $\sum_i \phi(u_i)$ \citep{deepsets}. A transformer-style in-context model can implement this computation by pooling token-wise features from training tokens into a summary representation using attention, and broadcasting the summary to each test token and applying a token-wise MLP. Combining these steps yields an in-context classifier whose test scores uniformly approximate those produced by the trained classifier. More details are available in \cref{app:proof_thm1_integer}

\textbf{Discussion.}
Proposition~\ref{prop:icl_contains_head} states that  ICL is at least as expressive as the standard
``freeze backbone and task-specific classifier'' evaluation protocol.
Any continuous trained-classifier pipeline induces a symmetric dataset-to-scores map on a bounded prompt domain, and such maps
can be uniformly approximated by an in-context model that processes the training set as context and the test set as queries,
producing all test predictions in a single forward pass.

\subsection{TIC-FM Emulates Gradient-Based Classifier Training}
\label{sec:ticfm_gd}

We provide a mechanistic justification for why a train-free in-context classifier can still behave like a
learning algorithm: the forward pass can emulate gradient-style updates in its activations.
This viewpoint is well-established in recent analyses of in-context learning in (linear) attention models \citep{transformerslearnincontextgradient,zhang2024incontext,xie2025incontextlearning}. We stress that this correspondence relies on idealized components (linear attention), the relationship should be read as an
approximate analogy rather than a strict equivalence.

\begin{proposition}[In-context emulation of gradient descent]
\label{prop:ticfm_gd}
Consider the scalar linear classifier trained on embeddings by one step of gradient descent on the squared loss
\[
\ell(W)=\frac{1}{2n_{\mathrm{tr}}}\sum_{i=1}^{n_{\mathrm{tr}}}\big(W z_i^{\mathrm{tr}}-y_i^{\mathrm{tr}}\big)^2,
\qquad W\in\mathbb{R}^{1\times q},
\]
with initialization $W^{(0)}=0$ and step size $\eta$.
Then there exists a TIC-FM instance whose in-context module $G_\theta$ is a linear
attention block such that, for every test token $r(z_j^{\mathrm{te}},\bot)$, the output scalar equals the $1$-step GD prediction:
\[
\big(G_\theta(\cdot)\big)^{(y)}_j \;=\; W^{(1)} z_j^{\mathrm{te}},\qquad \forall j=1,\dots,n_{\mathrm{te}}.
\]
\end{proposition}

\textbf{Proof sketch.}
Write GD updates in prediction space:
each step updates all query predictions by a linear combination of training residuals weighted by inner products
$\langle z_i^{\mathrm{tr}}, z_j^{\mathrm{te}}\rangle$.
A linear attention block computes these inner products via $QK^\top$ and aggregates residuals through its values,
thereby performing exactly one GD-style update on the label slot.
A full derivation is deferred to Appendix~\ref{app:proof_thm2}.

\textbf{Discussion.}
Proposition~\ref{prop:ticfm_gd} clarifies that ``train-free'' inference does not imply ``no learning'':
the learning dynamics of classifier training can be amortized into the forward computation of $G_\theta$.
While our TIC-FM is not restricted to linear regression, this toy equivalence provides a concrete explanation
for why conditioning on the labeled training set can produce optimization-like behavior at inference time.

\section{Experiments}
\label{sec:exp}

\subsection{Experimental setup}
\textbf{Datasets.} The UCR Time Series Archive \cite{dau2019ucr} is a widely recognized benchmark in the univariate time series classification domain. It comprises 128 datasets spanning diverse domains, such as human activity recognition, medical diagnosis, and intelligent transportation systems, making it an indispensable standard for evaluating the generalization and robustness of foundational models. In this study, unless stated otherwise, all experiments are conducted using the official UCR train/test splits.




\textbf{Compared Methods.} We compare against two recent state-of-the-art time series foundation models: Mantis \citep{mantis}  and MOMENT \citep{moment}. Mantis is an encoder-only model with 8M parameters pretrained with contrastive learning, and we use MOMENT-125M, an encoder–decoder model pretrained via masked reconstruction. Following prior evaluation protocols, we assess Mantis by training a Random Forest (RF) on frozen embeddings  and MOMENT  by training a Support Vector Machine (SVM). We report test accuracy averaged over full UCR datasets, using the standard train–test splits from \cite{dau2019ucr}. To control for classifier choice, we additionally evaluate both TSFMs using a shared set of trained classifiers, including a one-layer MLP, RF, and SVM, as well as training-free classifiers such as k-nearest neighbors (kNN) and nearest-centroid classification (NC) on the embeddings.


\textbf{Metrics.} We report classification accuracy (hereafter referred to as Acc) as the primary evaluation metric. Performance across datasets is summarized using average accuracy and average rank, to capture the overall performance trend. The higher the average accuracy, the better the performance; conversely, the smaller the average rank, the better the model. For each experiment, results are averaged over five random seeds by default. 

Further details on compared methods, implementation, and hardware specifications are provided in the Appendix ~\ref{app:experiments}.




\subsection{Comparative Analysis on Standard UCR}

In this section, we study the question: \textbf{Does our method outperform others under standard settings?}

\cref{StandardUCR} summarizes results on the UCR archive, and Appendix~\ref{app:mainsReluts} reports per-dataset scores. TIC-FM attains the highest average accuracy and the lowest mean rank among the compared methods, indicating that the gains are broadly distributed across datasets rather than concentrated in a small subset. Moreover, TIC-FM improves over the strongest frozen-backbone baseline with a trained classifier (Mantis+SVM) while requiring neither per-dataset classifier training nor hyperparameter tuning at deployment. This advantage stems from conditioning on the labeled context set, consistent with our in-context learning formulation, and suggests that improved use of labeled context drives the gains rather than task-specific classifier optimization.


\begin{table}[ht]
  \caption{\textbf{Classification Accuracy and Average Rank on full UCR Datasets.} The best results are in \textbf{bold}, and
the second best results are \underline{underlined}.} 
  \label{StandardUCR}
  \begin{center}
    \begin{small}
      \begin{sc}
        \begin{tabular*}{\columnwidth}{@{\extracolsep{\fill}}lrr}
          \toprule
           Method & Avg Acc  &  Avg Rank  \\
          \midrule
                MOMENT + RF & 77.51\% & 4.74 \\
                MOMENT + SVM & 77.98\%  & 3.95 \\
                MOMENT + MLP & 44.51\%  & 10.45 \\
                MOMENT + kNN & 75.72\%  & 5.90 \\
                MOMENT + NC & 66.06\% & 8.47 \\
                Mantis + RF & 78.67\% & 4.22 \\
                Mantis + SVM & \underline{79.06\%} & \underline{3.91} \\
                Mantis + MLP & 63.53\% & 8.58 \\
                Mantis + kNN & 77.07\%  & 5.17 \\
                Mantis + NC & 70.53\%  & 7.04 \\
                \textbf{TIC-FM} & \textbf{80.01\%} & \textbf{3.59} \\
          \bottomrule
        \end{tabular*}
      \end{sc}
    \end{small}
  \end{center}
  \vskip -0.1in
\end{table}

Additionally, \cref{StandardUCR} also highlights a limitation of the freeze backbone and task-specific classifier pipeline. With the backbone held fixed, performance varies substantially across classifiers, with SVM and RF consistently outperforming MLP. This sensitivity to the classifier suggests that reported results may be driven as much by the optimization behavior of the classifier as by the quality of the learned representations. TIC-FM mitigates this bias by introducing a context-based inference pipeline for time series classification that replaces per-dataset classifier training with inference-time conditioning on labeled support examples, thereby providing a more stable and fair evaluation of time-series foundation models on classification tasks.

\subsection{Robustness under Extreme Label Scarcity}
In this section, we study the question: \textbf{Is TIC-FM more reliable under extreme label scarcity than freeze backbone and classifier methods?}

We study an extreme low-shot regime by labeling only 10\% or 15\% of each dataset. To prevent potential leakage, we discard the official UCR training split and build the protocol solely from the official test split: we use stratified sampling to select 10\% (or 15\%) of the test set as the labeled context set, ensuring at least one labeled example per class.

As shown in Table~\ref{Low-shotdUCR}, TIC-FM attains higher average accuracy under both label fractions than the baselines. 
Under severe label scarcity, freeze backbone and classifier baselines vary substantially across classifier families: simple parametric classifiers (e.g., a one-layer MLP) may underfit or overfit, and often fail to convert additional labels into improved decision boundaries, whereas convex (SVM) or ensemble methods (RF) are typically more stable. Train-free rules such as kNN and NC are more robust because they avoid optimization, but their inductive bias is limited to distances or class means and cannot support richer task conditioning. In contrast, TIC-FM conditions on labeled samples via in-context inference, directly exploiting context–query interactions without per-dataset training, which accounts for its advantage in the low-label regime.



\begin{table}[ht]
  \caption{\textbf{Average Classification Accuracy on full UCR Datasets under Extreme Low-Shot Settings.} The best results are in \textbf{bold}, and
the second best results are \underline{underlined}. }
  \label{Low-shotdUCR}
  \begin{center}
    \begin{small}
      \begin{sc}
        \begin{tabular*}{\columnwidth}{@{\extracolsep{\fill}}lcc}
          \toprule
           Method & 10\%  & 15\%  \\
          \midrule
                MOMENT + RF  & 69.19\%  &  71.76\% \\
                MOMENT + SVM & 70.13\% &  72.88\%  \\
                MOMENT + MLP & 43.64\% &  43.84\%  \\
                MOMENT + kNN & 68.93\% &  71.99\%   \\
                MOMENT + NC  &  63.41\% &  64.85\%   \\
                Mantis + RF  & 69.57\% & 72.14\%     \\
                Mantis + SVM &  \underline{71.17\%} &  \underline{72.91\%}     \\
                Mantis + MLP &  56.77\% &   59.67\%      \\
                Mantis + kNN &  70.62\%&  72.80\%      \\
                Mantis + NC  &   66.07\% &   67.45\%     \\
            \textbf{TIC-FM} & \textbf{72.30\%} & \textbf{74.59\%}    \\
          \bottomrule
        \end{tabular*}
      \end{sc}
    \end{small}
  \end{center}
  \vskip -0.1in
\end{table}

\subsection{Scalability with respect to Supervision Budgets} 
\label{exp：Train-fraction}

In this section, we study the question: \textbf{How performance changes as more labeled examples become available, and whether TIC-FM remains competitive without any parameter updates.}

Since none of the models are pretrained on the UCR test split, we build this protocol solely from the official test split.
For each fraction $\alpha\in\{10\%,20\%,30\%,40\%,50\%,60\%\}$, we use stratified sampling to select $\alpha$ of the test examples as labeled context, and evaluate on the remaining $(1-\alpha)$.
We compare TIC-FM against Mantis and MOMENT using the probing heads from their original papers (RF for Mantis; SVM for MOMENT), and also report swapped heads (Mantis+SVM and MOMENT+RF) to control for classifier choice (further details are provided in Appendix~\ref{app:trainFra}).

\begin{figure}[ht]
  \centering
  \includegraphics[width=0.98\linewidth]{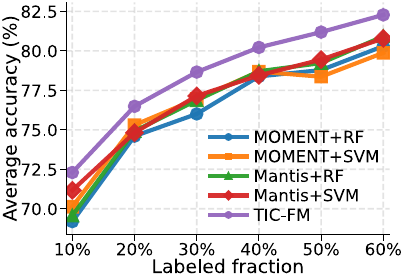}
  \caption{Scalability analysis with labeled data fractions.
  }
  \label{fig:train_fraction_scaling}
\end{figure}

As shown in \cref{fig:train_fraction_scaling}, TIC-FM attains higher average accuracy under both label fractions than the baselines. As the proportion of labeled examples increases from very limited supervision to more moderate budgets, the mean accuracy of TIC-FM improves steadily, indicating favorable scalability with respect to label availability. In contrast, freeze backbone and classifier baselines do not consistently exhibit monotonic or stable gains as supervision increases. For instance, MOMENT+SVM shows limited improvement and even noticeable fluctuations over certain labeling ranges. These results suggest that the in-context inference paradigm is more robust in leveraging additional labeled information. When labeling is relatively affordable and more labeled context examples can be provided, TIC-FM effectively absorbs the extra supervisory signal and continues to improve predictive performance. When labels are extremely scarce, TIC-FM still benefits from training-free conditioning on the available context set and strong generalization, leading to superior performance. Overall, TIC-FM performs robustly across a broad range of labeling budgets, making it a practical choice under varying annotation constraints.

\subsection{Impact of Context Window Size}
\label{sec:ctx_scaling}

In this section, we study the question: \textbf{How does TIC-FM scale with the context length when the query set is fixed?}

\begin{wrapfigure}[16]{r}{0.32\textwidth}
  \centering
    \includegraphics[width=0.335\textwidth]{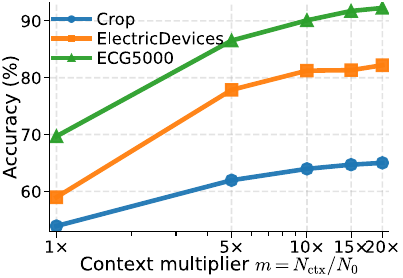}
  \caption{Impact of context length on inference accuracy. Increasing the number of context examples ($N_{ctx}$) consistently improves performance.}
  \label{fig:context_length_scaling}
\end{wrapfigure}
Unlike Section~\ref{exp：Train-fraction}, which varies the labeled fraction and evaluates on the remaining unlabeled portion, here we isolate the effect of longer contexts and test whether performance saturates. We focus on the three largest multiclass UCR datasets, Crop, ElectricDevices, and ECG5000. We first construct a fixed query set $Q$ by stratified sampling $10\%$ of $\mathcal{D}_{\mathrm{te}}$, ensuring at least one example per class, and keep $Q$ unchanged throughout. The remaining examples constitute a disjoint context pool $P=\mathcal{D}_{\mathrm{tr}} \cup (\mathcal{D}_{\mathrm{te}}\setminus Q)$. Let $C$ denote the number of classes and define the base context size as $N_0=10C$, corresponding to roughly ten labeled examples per class on average. We then vary the context budget as $N_{\mathrm{ctx}}=mN_0$ with $m\in\{1,5,10,15,20\}$, using class-balanced sampling from $P$ (additional details are provided in Appendix~\ref{app:ctx_scaling}).

As illustrated in \cref{fig:context_length_scaling}, TIC-FM exhibits a consistent monotonic improvement in accuracy as the context budget $N_{\mathrm{ctx}}$ expands from $N_0$ to $20N_0$, with the most pronounced gains occurring in the early low-data regime ($N_0 \rightarrow 5N_0$). This scaling behavior not only underscores the model's high sample efficiency but also empirically corroborates our theoretical analysis in \cref{sec:theory}; specifically, the performance trajectory mirrors that of an optimization algorithm benefiting from increased sample size, validating that TIC-FM successfully amortizes the learning dynamics into a single, data-driven forward pass without hitting early saturation.

\subsection{Ablation Study}
\label{sec:ablation}

We conduct an ablation study to quantify the contribution of individual components in TIC-FM. In each variant, we remove or disable exactly one module while keeping the backbone and evaluation protocol unchanged.

To assess the contribution of in-context inference, we replace the in-context classifier with a conventional trained classifier under the freeze backbone and task-specific classifier pipeline: we freeze the TIC-FM encoder, train an RF on embeddings extracted from the training split for each dataset, and evaluate on the corresponding test split. We denote this variant as \textbf{TIC-FM w/o ICL}.
To ablate inference-time ensembling, we remove cyclic label permutation and ensemble aggregation and instead perform a single forward pass, denoted as \textbf{TIC-FM w/o ensemble aggregation}.


\cref{fig:AblationStudy} summarizes the results. Removing the in-context learning classifier yields the largest drop, indicating that in-context inference accounts for most of TIC-FM’s gains: replacing it with a conventional classifier (RF) reduces the method to a freeze backbone and classifier pipeline and removes joint conditioning on labeled support at inference time. 
\begin{wrapfigure}[13]{r}{0.32\textwidth}
  \centering
  \clipbox{0pt 0pt 0pt 16pt}{%
     \includegraphics[width=\linewidth]{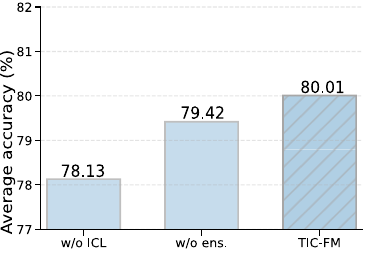}%
  }
  \caption{Ablation study on UCR. Removing the in-context classifier results in a marked decrease in performance.
  }
  \label{fig:AblationStudy}
\end{wrapfigure}
Disabling inference-time ensembling also degrades performance, but to a smaller extent, suggesting that cyclic label permutation and aggregation primarily provide additional robustness rather than being critical to the method. Overall, both components contribute: in-context inference delivers the main improvement, while ensemble aggregation consistently refines predictions by reducing sensitivity to the permutation choice, leading to the best average accuracy when combined.









\section{Conclusion}
{\looseness=-1
In this work, we revisit the commonly used ``zero-shot'' evaluation paradigm for classification with TSFMs, and argue that the widely adopted freeze backbone and task-specific classifier protocol is not truly training-free and is sensitive to the classifier. To address this limitation, we propose TIC-FM, an in-context classification framework with theoretical justification.
Empirically, TIC-FM achieves state-of-the-art performance on the full UCR archive and remains consistently stronger in extreme low-shot settings. 

More broadly, our training-free in-context formulation is particularly well-suited for scenarios where parameter updates are infeasible, such as on-device deployment with limited optimization budgets or privacy sensitive settings where data cannot be centralized.}


\section*{Impact Statement}
This paper presents work whose goal is to advance the field of machine learning. There are many potential societal consequences of our work, none of which we feel must be specifically highlighted here.

\bibliography{main}
\bibliographystyle{icml2026}

\newpage
\appendix
\onecolumn

\section{Details of TIC-FM Methodology}
\subsection{Pretraining}
\label{appendixPretrain}

We first pretrain the time series feature encoder $F_{\psi}$ (Section~\ref{TimeSerisesEncoder}) for 100 epochs on 100K synthetic time series generated by Cauker~\cite{cauker} using a contrastive objective. This pretraining encourages $F_{\psi}$ to learn discriminative representations by maximizing agreement between two stochastically augmented views of the same instance while reducing similarity across different instances.

Formally, let $\mathcal{B}=\{x_i\}_{i=1}^{B}$ denote a mini-batch of $B$ synthetic time series. For each $x_i$, we sample two augmentation operators $\mathcal{T}_1,\mathcal{T}_2 \sim \mathcal{T}$ (e.g., random cropping and resizing) to obtain two correlated views
$\tilde{x}_{i,1}=\mathcal{T}_1(x_i)$ and $\tilde{x}_{i,2}=\mathcal{T}_2(x_i)$.
The two views are encoded by $F_\psi$ and mapped to a contrastive latent space by a projection head $g_\phi(\cdot)$, yielding
\[
z_{i,1}=g_\phi\!\big(F_\psi(\tilde{x}_{i,1})\big),\qquad
z_{i,2}=g_\phi\!\big(F_\psi(\tilde{x}_{i,2})\big).
\]
In our implementation, $g_\phi$ is instantiated as a Layer Normalization followed by a linear projection. 

Following \cite{oord2018representation,he2020momentum}, we adopt a one-way in-batch InfoNCE loss, treating $z_{i,1}$ as queries and $z_{i,2}$ as keys. We first $\ell_2$-normalize the embeddings,
$\bar{z}_{i,1}=z_{i,1}/\|z_{i,1}\|_2$ and $\bar{z}_{i,2}=z_{i,2}/\|z_{i,2}\|_2$,
and compute the pairwise similarity logits
\[
s_{ij}=\frac{\bar{z}_{i,1}^{\top}\bar{z}_{j,2}}{\tau},\qquad \tau>0,
\]
forming a $B\times B$ logit matrix $S=[s_{ij}]$.
For each query index $i$, the positive key is the matched index $j=i$, and all $j\neq i$ serve as in-batch negatives. The resulting contrastive loss is the cross-entropy over the in-batch keys:
\[
\mathcal{L}_{\mathrm{con}}
=\frac{1}{B}\sum_{i=1}^{B}\left(-\log\frac{\exp(s_{ii})}{\sum_{j=1}^{B}\exp(s_{ij})}\right),
\]
equivalently $\mathcal{L}_{\mathrm{con}}=\mathrm{CE}(S,\,[1,2,\dots,B])$, where the target indices correspond to the diagonal alignment. Unless otherwise stated, we use $\tau=0.1$.

Regarding the optimization of the projection adapter $g_{\phi}$ (Section~\ref{sec:proj}) and the in-context classifier $G_\theta$ (Section~\ref{sec:icl}), specifically when the adapter is instantiated as an MLP, we employ a two-stage training protocol. In the first stage, we pretrain the in-context classifier on synthetic datasets generated from structural causal models (SCMs) for 27050 steps, following the Orion-MSP methodology~\cite{orion}. Subsequently, in the second stage, we freeze both the time series feature encoder and the pretrained in-context classifier. The projection adapter is then trained for 5 epochs on the UCR training splits using a standard cross-entropy objective.  Specifically, we adopt an episodic training paradigm: for each iteration, we construct a classification task by sampling a labeled context set and a batch of query examples from a dataset's training split. 
The model predicts query labels conditioned on the context, and the loss is computed on these predictions. 
All gradients are backpropagated exclusively to the adapter parameters. 
Crucially, no samples from the UCR test sets are accessed during this process, thereby preventing any data leakage in our experiments.


To accommodate varying feature dimensions, we introduce RowMixerLite, a Transformer-based projection adapter that treats the feature dimension as a token sequence, enabling a unified interface for downstream in-context classification across heterogeneous time-series embedding dimensionalities. Concretely, given an input representation, RowMixerLite partitions the feature axis into non-overlapping patches of size $8$ and applies a shared patch projection to map each patch into a token of dimension $d_{\text{model}}{=}128$.   The resulting patch tokens are augmented with a small set of learnable special tokens, consisting of $4$ class tokens and $2$ global tokens, each of dimension $128$. These tokens provide dedicated aggregation slots that summarize patch-level evidence and improve the stability of the subsequent pooling operation. We then apply a $3$ layer Transformer encoder over the token sequence, using $8$ attention heads, feed-forward dimension $256$, pre-norm, and dropout $0.0$. The output tokens are finally arranged into $4$ class-specific embeddings, which are concatenated to match the token dimensionality expected by the in-context classifier, yielding $d_{\text{icl}}{=}512$.

Regarding the optimization of RowMixerLite and the in-context classifier, we jointly pretrain them end-to-end for 15750 steps on SCM-based synthetic data generated by Orion-MSP, which is designed to mimic diverse input-representation distributions. To improve robustness and reduce sensitivity to feature ordering, we apply feature shuffling with probability $0.25$ during training. After pretraining, we integrate the RowMixerLite adapter with our time series encoder  $F_{\psi}$ and the in-context classifier $G_\theta$   for time series classification. This fully synthetic model supports multivariate time series classification and achieves competitive performance on UCR, reaching \textbf{79.75\%} average accuracy. These results suggest that the TIC-FM framework can be trained entirely on synthetic data, without any real datasets, while remaining competitive.

\subsection{Inference Details}
\label{app:inference}

\subsubsection{Hierarchical class extension for $K > C_{\max}$}
\label{app:hierarchy}
To address scenarios where the total number of classes $K$ exceeds the model's architectural limit $C_{\max}$ (constrained by the pre-defined label embedding dimension or context window), we implement a hierarchical class-extension strategy following \cite{tabicl}. 
Instead of simple truncation, this method dynamically constructs a classification tree $\mathcal{T}$ derived from the labeled support set $\mathcal{D}_{\mathrm{tr}}$, allowing TIC-FM to perform inference over an arbitrary number of classes without parameter updates.

\paragraph{Tree Construction (Fit Phase).}
The construction process proceeds recursively starting from the root. Let $\mathcal{D}_{\mathcal{N}} = \{(x_i, y_i)\}$ denote the subset of support samples reaching node $\mathcal{N}$, and $\mathcal{Y}_{\mathcal{N}}$ be the set of unique classes present in $\mathcal{D}_{\mathcal{N}}$. The tree is built based on the following logic:
\begin{enumerate}
    \item \textbf{Leaf Condition:} If $|\mathcal{Y}_{\mathcal{N}}| \le C_{\max}$, the node $\mathcal{N}$ is designated as a \emph{leaf}. It directly stores $\mathcal{D}_{\mathcal{N}}$ as its local in-context demonstrations for standard inference.
    \item \textbf{Internal Node Splitting:} If $|\mathcal{Y}_{\mathcal{N}}| > C_{\max}$, the node becomes an internal router. We partition the classes $\mathcal{Y}_{\mathcal{N}}$ into $G$ disjoint groups $\{\mathcal{G}_1, \dots, \mathcal{G}_G\}$, where $G = \lceil |\mathcal{Y}_{\mathcal{N}}| / C_{\max} \rceil$. The grouping strategy ensures balanced class distribution across branches.
    \item \textbf{Label Coarsening (Meta-Task Construction):} We construct a coarse-grained classification task for the internal node. The original labels $y_i$ in $\mathcal{D}_{\mathcal{N}}$ are mapped to their corresponding group indices $g(y_i) \in \{0, \dots, G-1\}$. This forms a meta-support set $\mathcal{D}'_{\mathcal{N}} = \{(x_i, g(y_i))\}$, which serves as the context for deciding which branch to traverse.
    \item \textbf{Recursion:} We instantiate $G$ child nodes, where the $j$-th child is recursively fitted using only the subset of data belonging to group $\mathcal{G}_j$.
\end{enumerate}

\paragraph{Recursive Inference (Predict Phase).}
During inference, a query sample $x_{\mathrm{te}}$ traverses the tree from the root. 
The probability of a final class $y$ is computed via the chain rule of probability along the path from the root to the leaf containing $y$.
For an internal node $\mathcal{N}$, the model behaves as a router, predicting the probability distribution over groups $P(g \mid x_{\mathrm{te}}; \mathcal{D}'_{\mathcal{N}})$ using the meta-support set.
Algorithm~\ref{alg:hierarchical_predict} formalizes this recursive probability aggregation.

\begin{algorithm}[h]
   \caption{Hierarchical In-Context Inference}
   \label{alg:hierarchical_predict}
\begin{algorithmic}[1]
   \STATE {\bfseries Input:} Query sample $x$, Current Node $\mathcal{N}$, Model $\Phi$
   \STATE {\bfseries Output:} Probability distribution over classes in $\mathcal{N}$
   
   \IF{$\mathcal{N}$ is Leaf}
       \STATE \COMMENT{Standard ICL inference with local support set}
       \STATE Let $(\mathbf{R}, \mathbf{y})$ be the support set stored in $\mathcal{N}$
       \STATE Return $\Phi(x \mid \mathbf{R}, \mathbf{y})$
   \ELSE
       \STATE \COMMENT{Router step: predict group probabilities}
       \STATE Let $(\mathbf{R}, \mathbf{g})$ be the meta-support set (labels are group indices)
       \STATE $\mathbf{P}_{\text{group}} \gets \Phi(x \mid \mathbf{R}, \mathbf{g})$ \hfill $\triangleright$ Shape: $[G]$
       
       \STATE $\mathbf{P}_{\text{final}} \gets \mathbf{0}$
       \FOR{each group index $j \in \{0, \dots, G-1\}$}
           \STATE Let $\mathcal{C}_j$ be the $j$-th child node
           \STATE $\mathbf{P}_{\text{child}} \gets \text{Call Self}(x, \mathcal{C}_j, \Phi)$ \hfill $\triangleright$ Recursive call
           \STATE \COMMENT{Accumulate probability mass: $P(y) = P(y|g_j) \cdot P(g_j)$}
           \STATE $\mathbf{P}_{\text{final}}[\text{indices of } \mathcal{C}_j] \gets \mathbf{P}_{\text{child}} \times \mathbf{P}_{\text{group}}[j]$
       \ENDFOR
       \STATE Return $\mathbf{P}_{\text{final}}$
   \ENDIF
\end{algorithmic}
\end{algorithm}

\subsubsection{Test-time ensembling via cyclic label permutations}
\label{app:inference_ensemble}

\paragraph{Ensemble members.}
Let $K$ be the number of classes in the current dataset.
We construct an ensemble of $M=n_{\text{est}}$ members, each parameterized by a cyclic label-shift offset
$o_m \in \{0,\ldots,K-1\}$.
In practice, we generate offsets by shuffling $\{0,\ldots,K-1\}$ and cycling through the list if $M>K$.

\paragraph{Cyclic label permutation and aggregation.}
For member $m$ with offset $o_m$, we apply a cyclic permutation \emph{only} to the support labels:
\begin{equation*}
\tilde y_{\mathrm{tr}} = \pi_{o_m}(y_{\mathrm{tr}}), \qquad
\pi_{o_m}(y) = (y + o_m)\bmod K.
\end{equation*}
We then run the ICL model using $(X_{\mathrm{all}}, \tilde y_{\mathrm{tr}})$ to obtain the query logits
$\ell_m \in \mathbb{R}^{N_{\mathrm{te}}\times K}$ in the permuted label space.
To aggregate predictions across members, we map logits back to the original label space via the inverse permutation
$\pi_{o_m}^{-1}$ (implemented as a circular shift along the class dimension) and average:
\begin{equation*}
\bar{\ell}=\frac{1}{M}\sum_{m=1}^{M} \pi_{o_m}^{-1}(\ell_m).
\end{equation*}

\section{Proofs for Section~\ref{sec:theory}}
\subsection{Proof of Proposition~\ref{prop:icl_contains_head}}
\label{app:proof_thm1_integer}

We prove that an in-context classifier can uniformly approximate the trained-classifier score map on the compact
padded-and-masked domain $\mathcal{X}$. By Assumption and the continuity of $r(\cdot,\cdot)$, the induced padded-and-masked prompt-token blocks $(U,V)$ lie in a compact set $\mathcal{X}\subset (\mathbb{R}^d)^N\times(\mathbb{R}^d)^M$.

\subsubsection{Step 1: Reduce to approximating a continuous invariant score map}

\textbf{Prompt tokenization.}
Let $g_\phi:\mathbb{R}^q\to\mathbb{R}^d$ be the projection adapter and
let $E_y:\{0,\ldots,C_{\max}-1\}\to\mathbb{R}^d$ be the label embedding.
Extend it by $E_y(\bot)\triangleq 0\in\mathbb{R}^d$ and define
\[
r(z,y)\triangleq g_\phi(z)+E_y(y)\in\mathbb{R}^d.
\]
In this proposition, the label space size equals the number of classes, i.e., $C_{\max}=K$.

Pad to fixed sizes $(N,M)$ and denote the resulting token blocks by $U=(u_1,\dots,u_N)$ and $V=(v_1,\dots,v_M)$,
where $u_i=r(z_i^{\mathrm{tr}},y_i^{\mathrm{tr}})$ and $v_j=r(z_j^{\mathrm{te}},\bot)$ for unmasked tokens.
For padded positions, set $u_i=u_{\mathrm{pad}}$ for $i>n_{\mathrm{tr}}$ and $v_j=v_{\mathrm{pad}}$ for $j>n_{\mathrm{te}}$,
with fixed $u_{\mathrm{pad}},v_{\mathrm{pad}}\in\mathbb{R}^d$; these are ignored by the padding/masking convention.

The trained-classifier pipeline induces a score map
\[
F(U,V)\ \triangleq\ \big(F_1(U,V),\dots,F_M(U,V)\big)\ \in\ (\mathbb{R}^K)^M,
\]
where $F_j(U,V)$ equals $h_\tau(z_j^{\mathrm{te}})$ for unmasked test tokens, with $h_\tau=\mathsf{Train}(\mathcal{D}_{\mathrm{tr}})$.
By assumption, $F$ is continuous on $\mathcal{X}$, permutation-invariant in the training blocks $U$,
and permutation-equivariant in the test blocks $V$.

Thus it suffices to approximate $F$ uniformly on $\mathcal{X}$ by a model acting on the prompt tokens.

\subsubsection{Step 2: Symmetric polynomials are dense for continuous training-invariant maps \citep{permutationinvariantequivariant}}

We first approximate each scalar coordinate of $F$ by a polynomial that is symmetric in the training blocks.
Let $F_{j,k}(U,V)$ be the $k$-th coordinate of $F_j(U,V)$.

\begin{lemma}[Density of polynomials symmetric in training blocks]
\label{lem:sym_poly_dense_uv}
Let $\mathcal{X}\subset (\mathbb{R}^{d})^{N}\times(\mathbb{R}^{d})^{M}$ be compact.
For any continuous function $h:\mathcal{X}\to\mathbb{R}$ that is invariant under permutations of the $N$ training blocks,
and any $\varepsilon>0$, there exists a polynomial $p(U,V)$ such that
(i) $p$ is symmetric in the training blocks $U$, and
(ii) $\sup_{(U,V)\in\mathcal{X}} |h(U,V)-p(U,V)|\le \varepsilon$.
\end{lemma}

\begin{proof}
By Stone--Weierstrass, ordinary polynomials in all coordinates of $(U,V)$ are dense in $\mathcal{C}(\mathcal{X})$.
Let $q(U,V)$ be a polynomial with $\sup_{\mathcal{X}}|h-q|\le \varepsilon$.
Define its symmetrization over the training blocks:
\[
\mathrm{Sym}(q)(U,V)\ \triangleq\ \frac{1}{N!}\sum_{\pi\in S_N} q(\pi\cdot U, V).
\]
Then $\mathrm{Sym}(q)$ is a polynomial symmetric in $U$.
Since $h(\pi\cdot U,V)=h(U,V)$, we have
$
|h(U,V)-\mathrm{Sym}(q)(U,V)|\le \sup_{\mathcal{X}}|h-q|\le \varepsilon
$
for all $(U,V)\in\mathcal{X}$.
\end{proof}

Apply Lemma~\ref{lem:sym_poly_dense_uv} to each $F_{j,k}$ and take a union bound over finitely many coordinates.
For any $\varepsilon>0$, there exists a map $P(U,V)=(P_1,\dots,P_M)\in(\mathbb{R}^K)^M$ such that:
(i) each $P_j$ is a polynomial in $(U,V)$ symmetric in $U$, and
(ii)
\[
\sup_{(U,V)\in\mathcal{X}}\ \max_{1\le j\le M}\ \|F_j(U,V)-P_j(U,V)\|_\infty
\le \varepsilon/2.
\]

\subsubsection{Step 3: Symmetric polynomials reduce to sums of elementwise features \citep{deepsets}}

Fix $j$ and consider $P_j(U,V)$ as a polynomial in $(u_1,\dots,u_N,v_j)$ that is symmetric in $(u_1,\dots,u_N)$.
Since $P_j$ has finite total degree, there exists a finite set of monomials
$\{m_{\alpha_r}\}_{r=1}^s$ on $\mathbb{R}^{d}$ such that $P_j$ can be written as a polynomial in the aggregated monomials
$\sum_{i=1}^N m_{\alpha_r}(u_i)$, together with $v_j$:
\[
P_j(U,v_j)=\widetilde P_j\Big(\sum_{i=1}^N m_{\alpha_1}(u_i),\dots,\sum_{i=1}^N m_{\alpha_s}(u_i),\ v_j\Big).
\]
Define
\[
\phi(u)\triangleq \big(m_{\alpha_1}(u),\dots,m_{\alpha_s}(u)\big)\in\mathbb{R}^s,
\qquad
\rho_j(s,v)\triangleq \widetilde P_j(s,v)\in\mathbb{R}^K.
\]
Then
\[
P_j(U,v_j)=\rho_j\Big(\sum_{i=1}^N \phi(u_i),\ v_j\Big),
\]
which is a DeepSets-style form \citep{settransformerframeworkattentionbased,universalapproximatorssequencetosequence}.

\subsubsection{Step 4: Realize the DeepSets computation with an in-context model}

It remains to show that a transformer-like in-context model can implement the map
\[
v_j\ \mapsto\ \rho_j\Big(\sum_{i=1}^N \phi(u_i),\ v_j\Big)
\quad\text{uniformly on }\mathcal{X}.
\]

\begin{lemma}[Masked pooling of training features]
\label{lem:pooling_masked_sum}
There exist parameters of a self-attention layer (with masking) and a designated summary token $s$
such that, after applying a token-wise MLP implementing $\phi$ on training tokens,
the summary token can represent the masked sum
\[
s^{(1)}\approx \sum_{i=1}^{n_{\mathrm{tr}}}\phi(u_i).
\]
\end{lemma}

\begin{proof}
Use a token-wise MLP to produce values $V_i=\phi(u_i)$.
Choose keys/queries so that the summary token attends uniformly to unmasked training tokens and ignores masked ones.
This yields a masked average; multiplying by the (mask-computable) count converts it to a masked sum.
\end{proof}

\begin{lemma}[Broadcast and apply a continuous map]
\label{lem:broadcast_and_apply}
There exists a second attention layer in which each query token $v_j$ attends to the summary token $s^{(1)}$ and receives
the pooled vector. A subsequent token-wise MLP can uniformly approximate the continuous map
$(s^{(1)},v_j)\mapsto \rho_j(s^{(1)},v_j)$ on the compact domain.
\end{lemma}

\begin{proof}
Set attention logits so each $v_j$ attends primarily to the summary token.
Since $\rho_j$ is continuous on a compact set, standard universal approximation guarantees that an MLP can approximate it
uniformly.
\end{proof}

\begin{proof}[Completion of the proof]
By Steps 2--3, there exists a DeepSets-form map $G^\star$ such that
$\sup_{\mathcal{X}}\max_j \|F_j-G^\star_j\|_\infty\le \varepsilon/2$.
By Lemmas~\ref{lem:pooling_masked_sum}--\ref{lem:broadcast_and_apply}, there exists an in-context model $G_\theta$ such that
$\sup_{\mathcal{X}}\max_j \|G^\star_j-G_\theta(\cdot)_j\|_\infty\le \varepsilon/2$.
By the triangle inequality,
$
\sup_{\mathcal{X}}\max_j \|F_j-G_\theta(\cdot)_j\|_\infty\le \varepsilon,
$
which proves the score-approximation claim in Proposition~\ref{prop:icl_contains_head}.
\end{proof}
\begin{corollary}[Label matching under a uniform margin]
Let $h_\tau = \mathrm{Train}(\mathcal{D}_{\mathrm{tr}})$ and define the trained-classifier predicted label for each test embedding as
\[
y^{\star}_j \;=\; \arg\max_{k\in[K]} \big(h_\tau(z^{\mathrm{te}}_j)\big)_k,\qquad j=1,\dots,n_{\mathrm{te}}.
\]
Assume the trained-classifier scores admit a \emph{uniform margin} on $\mathcal X$, i.e., there exists $\gamma>0$ such that
for all $(\mathcal{D}_{\mathrm{tr}},\{z^{\mathrm{te}}_j\}_{j=1}^{n_{\mathrm{te}}})\in\mathcal X$ and all $j$,
\[
\big(h_\tau(z^{\mathrm{te}}_j)\big)_{y^{\star}_j}
\;-\;
\max_{k\neq y^{\star}_j}\big(h_\tau(z^{\mathrm{te}}_j)\big)_k
\;\ge\;\gamma.
\]
If $G_\theta$ satisfies the score-approximation bound in Proposition~5.2 with some $\varepsilon<\gamma/2$, and we define
\[
\hat y_j \;=\; \arg\max_{k\in[K]} \big(G_\theta(\cdot)_j\big)_k,
\]
then $\hat y_j = y^\star_j$ for all $j=1,\dots,n_{\mathrm{te}}$ on $\mathcal X$.
\end{corollary}

\paragraph{Interpretation.}
The above result is an expressivity universal approximation statement in a continuous Euclidean token space.
\cref{prop:icl_contains_head} should be interpreted as an existence claim, there exist parameters $\theta$
such that an in-context model can approximate the target score map uniformly on $X$.
It does not imply that a particular training algorithm will necessarily recover such parameters, nor does it provide unconditional
guarantees for the realized optimization dynamics in practice.

\subsection{Proof of Proposition~\ref{prop:ticfm_gd}}
\label{app:proof_thm2}

We prove that $T$ stacked linear-attention blocks can emulate $T$ steps of GD for a scalar linear head, following the
prediction-space derivation in \citet{transformerslearnincontextgradient,zhang2024incontext,xie2025incontextlearning}.

\subsubsection{Step 1: GD in prediction space}

Let $\hat y^{(t)}(z)=W^{(t)}z$ with $W^{(0)}=0$.
Define training residuals $e_i^{(t)}=\hat y^{(t)}(z_i^{\mathrm{tr}})-y_i^{\mathrm{tr}}$.
One GD step gives
\[
W^{(t+1)} = W^{(t)} - \frac{\eta}{n_{\mathrm{tr}}}\sum_{i=1}^{n_{\mathrm{tr}}} e_i^{(t)} (z_i^{\mathrm{tr}})^\top.
\]
Multiplying by a query $z$ yields the prediction update
\[
\hat y^{(t+1)}(z)=\hat y^{(t)}(z)-\frac{\eta}{n_{\mathrm{tr}}}\sum_{i=1}^{n_{\mathrm{tr}}} e_i^{(t)}\langle z_i^{\mathrm{tr}}, z\rangle.
\]
Thus each step updates all query predictions using training residuals weighted by inner products.

\subsubsection{Step 2: One linear-attention block implements one prediction update}

Consider tokens whose first $q$ coordinates store $z$ and whose last coordinate stores the current scalar prediction.
A linear attention block computes $\langle z_i^{\mathrm{tr}}, z\rangle$ via $QK^\top$ (with suitable projections),
uses values to carry $e_i^{(t)}$, and writes the aggregated update into the label slot, realizing the map
\[
\hat y^{(t)}(z)\ \mapsto\ \hat y^{(t+1)}(z)
\]
for all queries in parallel. The block parameters are fixed and do not depend on $t$. Exact parameters can be found in \citet{transformerslearnincontextgradient}.

\subsubsection{Step 3: Stacking $T$ blocks potentially equals $T$ GD steps \citep{looped1,looped2}}

Applying Step 2 repeatedly for $T$ blocks yields $\hat y^{(T)}(z)=W^{(T)}z$ for every query embedding $z$.
This proves Proposition~\ref{prop:ticfm_gd}.

\paragraph{Limitations.}
The connection between in-context learning and gradient descent is an active research direction.
Our proof is mechanistic: it shows that, under an linear attention construction.
In particular, our analysis does not cover standard softmax attention, normalization or general nonlinear heads. 


\section{Additional Details and Reproducibility}
\label{app:experiments}

\paragraph{Roadmap.}
This appendix provides (i) detailed descriptions of baselines (App.~\ref{app:baselines});
(ii) evaluation protocols for supervision budget scalability and context window size analysis(App.~\ref{app:protocols});
(iii) implementation and hardware specifications (App.~\ref{app:imple}); 
(IV) additional experimental results(App.~\ref{app:additional_results}).

\subsection{Compared Methods}
\label{app:baselines}

\subsubsection{TSFMs}

\paragraph{Mantis.}
Mantis is a lightweight foundation model tailored for time-series classification. It adopts an encoder-only architecture that first converts an input sequence into a fixed set of patch-level tokens via a token generator that combines convolutional patching, a differential branch, and patch-wise statistical encoding, and then applies a ViT-style Transformer encoder with a learnable class token to aggregate token information. The model is pretrained in a self-supervised manner using a contrastive objective, where two stochastic augmentations of the same series form a positive pair and other samples in the batch serve as negatives. In our experiments, we use the 8M-parameter variant of Mantis and initialize it with publicly available pretrained weights.\footnote{\url{https://huggingface.co/paris-noah/Mantis-8M/tree/main}}

\paragraph{MOMENT.} MOMENT is a Transformer-based time-series foundation model pretrained via masked time-series modeling. It partitions a univariate series into non-overlapping patches, embeds them, and randomly masks a subset using a dedicated mask token. A Transformer encoder produces contextualized patch representations, which are fed into a lightweight reconstruction head trained to recover the masked patches under an MSE objective. MOMENT is pretrained on the Time Series Pile, a curated collection covering forecasting, classification, and anomaly-detection datasets (e.g., Informer-style long-horizon benchmarks, the Monash forecasting archive, UCR/UEA classification datasets, and the TSB-UAD anomaly benchmark), using only training splits to mitigate potential contamination. We initialize MOMENT with publicly available pretrained weights.\footnote{\url{https://huggingface.co/AutonLab/MOMENT-1-base/tree/main}}


\subsubsection{Classifiers on Frozen Embeddings}
\label{app:probe_heads}
\paragraph{Frozen embeddings.}
For ``a frozen encoder paired with a task-specific classifier'' evaluations, we keep the TSFM backbone frozen and extract embeddings for all instances once.
Let $z=f(x)\in\mathbb{R}^{d}$ denote the embedding of a time series $x$, and let
$Z_{\mathrm{tr}}=\{(z_i^{\mathrm{tr}},y_i^{\mathrm{tr}})\}_{i=1}^{N_{\mathrm{tr}}}$ and
$Z_{\mathrm{te}}=\{z_j^{\mathrm{te}}\}_{j=1}^{N_{\mathrm{te}}}$ be the resulting train/test embeddings.
Unless otherwise stated, we do not apply feature standardization/normalization to the embeddings before fitting any classifier.

\paragraph{Trained classifiers.}
We consider the following lightweight supervised classifiers trained on $Z_{\mathrm{tr}}$:
\textbf{(i) Random Forest (RF)}, an ensemble of decision trees;
\textbf{(ii) Linear SVM}, a max-margin linear classifier; and
\textbf{(iii) one-layer MLP}, implemented as a single linear layer $g_\omega(z)=Wz+b$ trained with cross-entropy.

\paragraph{Train-free classifiers.}
We also report two non-parametric classifiers computed directly in the embedding space:
\textbf{(i) 1-NN}, which predicts by the nearest training embedding under Euclidean distance,
$\hat y(z)=y_{i^\star}$ where $i^\star=\arg\min_i \|z-z_i^{\mathrm{tr}}\|_2$; and
\textbf{(ii) Nearest Centroid (NC)}, which computes class centroids
$\mu_c=\frac{1}{|I_c|}\sum_{i\in I_c} z_i^{\mathrm{tr}}$ and predicts
$\hat y(z)=\arg\min_c \|z-\mu_c\|_2$.


\paragraph{Classifier-specific hyperparameters and training details}

We use the exact classifier configurations implemented in our code. Except for the one-layer MLP (our PyTorch implementation), all other classifiers are instantiated via \texttt{scikit-learn} and trained/fitted on the training split only. Unless explicitly stated below, we keep \texttt{scikit-learn} hyperparameters at their default values. Specifically, Random Forest (RF) uses \texttt{n\_estimators}=100. Linear SVM uses \texttt{max\_iter}=100 (with all remaining parameters set to \texttt{scikit-learn} defaults). The one-layer MLP is trained with Adam (learning rate $10^{-3}$, weight decay $0$) for 10 epochs using batch size 256. For non-parametric baselines, 1-NN uses $k=1$, and Nearest Centroid (NC) follows the standard nearest-centroid decision rule. In all cases, classifiers are fit using representations extracted from the frozen backbone on the training set, and evaluation is performed on the test set without any access to test labels.

\subsubsection{Evaluation Details}

\paragraph{Overview.}
We evaluate all methods under the freeze backbone and classifier pipeline, as described in the main text, to ensure a fair comparison between trained and train-free classifiers. For feature-based baselines, we obtain pretrained checkpoints of {Mantis} and {MOMENT} from the authors' official public repositories. Unless otherwise stated, we {freeze all backbone parameters} and use each backbone solely as a feature extractor. Given a dataset-specific train/test split, we first compute time-series representations for all samples using the frozen backbone, and then fit a downstream classifier on the training representations only. Final performance is reported on the test split. 

\paragraph{Backbone acquisition and freezing.}
For {Mantis} and {MoMENT}, we use the authors' released pretrained models and the accompanying preprocessing pipelines provided in their repositories. We do not perform any additional finetuning of these backbones on the evaluation datasets. In all experiments, the backbone is set to evaluation mode, and gradients are disabled, so that the extracted representations are deterministic given the input and random seed.

\paragraph{Representation extraction.}
For each dataset, we transform every time series $x$ into a fixed-dimensional embedding $z \in \mathbb{R}^q$ using the frozen backbone:
\[
z = F(x),
\]
where $F$ denotes either {Mantis} or {MoMENT}. We extract embeddings for both the training split and the test split. To avoid any information leakage, {all hyperparameters of the downstream classifier are selected without accessing the test split}, and no statistics computed on the test split are used during training (e.g., normalization parameters are computed on the training split and then applied to the test split if needed).

\paragraph{Downstream classifiers.}
On top of the extracted embeddings, we evaluate two categories of classifiers:

\begin{itemize}
    \item \textbf{Trained classifiers.} These methods fit a task-specific classifier on the training embeddings $\{(z_i^{tr}, y_i^{tr})\}$. The classifier is optimized using only the training split of the corresponding dataset. The trained classifier is then applied to test embeddings $\{z_j^{te}\}$ to produce predictions.

    \item \textbf{Train-free classifiers.} These methods do not perform gradient-based parameter updates at evaluation time. Instead, they produce predictions for each test instance by conditioning on the labeled training embeddings (e.g., via nearest-neighbor style matching, prompt-based inference, or other non-parametric/ICL-style mechanisms), while keeping all model parameters fixed.
\end{itemize}

Across both categories, the backbone remains frozen and is {never} updated on any evaluation dataset.

\paragraph{Training data usage and leakage prevention.}
For every dataset, any fitting/training procedure for downstream classifiers uses {only} the training split of that dataset. The test split is used exclusively for final evaluation. In particular, when a method conditions on the labeled training set (e.g., train-free prompting/ICL), the labeled set is always the dataset's training split, and no test labels are used at any point.

\paragraph{Random seeds and reporting.}
To account for stochasticity in downstream training (e.g., classifier initialization, minibatch ordering) and any stochastic components in the evaluation pipeline, we run each method on each dataset with \textbf{five} different random seeds. We report the \textbf{mean test accuracy} across these five runs as the final result for that dataset:
\[
\mathrm{Acc} = \frac{1}{5}\sum_{s=1}^{5} \mathrm{Acc}^{(s)}.
\]
When applicable, we keep all non-essential factors fixed across seeds (e.g., backbone checkpoint, preprocessing settings) so that the reported variance reflects only the intended sources of randomness.

\subsection{Evaluation Protocols}
\label{app:protocols}
\subsubsection{Supervision Budget Scaling Protocol} 
\label{app:train_fraction}

\paragraph{Purpose.}
This protocol evaluates how prediction performance scales with the amount of {labeled context} available at inference time. To avoid any interaction with the official UCR training split in this study, we construct the entire scaling protocol {solely from the official UCR test split}.

\paragraph{Pool construction from the official UCR test split.}
For each UCR dataset $\mathcal{D}$, we first load the official test set
\[
\mathcal{T}=\{(x_i,y_i)\}_{i=1}^{N},
\]
and treat it as an unlabeled pool from which we derive both the labeled context set and the query set.
For each train-fraction (supervision budget) $\alpha \in \{0.1,0.2,0.3,0.4,0.5,0.6\}$, we construct a disjoint split
\[
\mathcal{T} = \mathcal{C}_\alpha \ \cup\ \mathcal{Q}_\alpha,\qquad
\mathcal{C}_\alpha \cap \mathcal{Q}_\alpha = \emptyset,
\]
where $\mathcal{C}_\alpha$ is the labeled context set and $\mathcal{Q}_\alpha$ is the evaluation query set.

\paragraph{Stratified splitting with minimum one per class.}
Let $\mathcal{Y}$ be the set of class labels appearing in $\mathcal{T}$, and let $K=|\mathcal{Y}|$.
Our splitter enforces (whenever possible) that the context set contains at least one example per class:
\[
\forall c\in\mathcal{Y},\quad |\{(x,y)\in\mathcal{C}_\alpha: y=c\}|\ge 1.
\]
Implementation-wise, we perform a two-stage stratified sampling procedure:
\begin{enumerate}
    \item \textbf{Mandatory coverage.} For each class $c\in\mathcal{Y}$, we uniformly sample one index from the class and place it into $\mathcal{C}_\alpha$.
    \item \textbf{Proportional fill.} If additional context samples are needed, we allocate the remaining budget across classes in proportion to the remaining per-class counts and sample without replacement. Leftover slots (due to rounding) are assigned by largest fractional parts, with a final fallback that assigns remaining slots to any class with available samples.
\end{enumerate}
We also enforce that the query set is non-empty by capping $N_{\mathrm{ctx}}\le N-1$.
If $N_{\mathrm{ctx}}<K$ under a very small $\alpha$, we set $N_{\mathrm{ctx}}:=K$ (still capped by $N-1$) so that the minimum-one-per-class constraint remains feasible.
After sampling, we deterministically shuffle indices within $\mathcal{C}_\alpha$ and $\mathcal{Q}_\alpha$ for bookkeeping.

\paragraph{Seed control and reproducibility.}
We decouple the randomness of {data splitting} from the randomness of {downstream training}.
Each run is parameterized by:
(i) a \texttt{split seed} $s_{\mathrm{split}}$ controlling the stratified split of $\mathcal{T}$ into $(\mathcal{C}_\alpha,\mathcal{Q}_\alpha)$, and
(ii) five \texttt{run seeds} $\{s_r\}_{r=1}^{5}$ controlling the stochasticity of downstream classifier training.
In our implementation, if the run seeds are not explicitly provided, they are generated by drawing five integers from a NumPy RNG initialized with a base seed.
Importantly, for a fixed dataset and $\alpha$, we keep the split $(\mathcal{C}_\alpha,\mathcal{Q}_\alpha)$ \emph{fixed across the five run seeds} to isolate variance due to classifier training (e.g., initialization and minibatch order).

\paragraph{Label re-indexing and class consistency.}
For every split, we re-index class labels according to the context labels:
we map the unique labels in $\mathcal{C}_\alpha$ to $\{0,\dots,K-1\}$ and apply the same mapping to $\mathcal{Q}_\alpha$.
If any query label is not present in the context label set, the run is invalid; in code this triggers an error.
In practice, the minimum-one-per-class constraint prevents such failures except for degenerate cases where some class appears only once in the pool (in which case it must belong to either $\mathcal{C}_\alpha$ or $\mathcal{Q}_\alpha$ under a disjoint split). We log these cases as split warnings.

\paragraph{Reporting.}
For each dataset $\mathcal{D}$ and fraction $\alpha$, we run five trials indexed by run seeds $\{s_r\}_{r=1}^{5}$ and compute accuracy on the query set $\mathcal{Q}_\alpha$.
We report the mean accuracy across the five trials as the final result for $(\mathcal{D},\alpha)$.

\subsubsection{Context Window Scaling Protocol}
\label{app:ctx_scaling}

\paragraph{Testbeds.}
Due to the large sample requirements, we use the three largest UCR datasets: Crop, ElectricDevices, and ECG5000.
Let $\mathcal{D}_{\mathrm{tr}}$ and $\mathcal{D}_{\mathrm{te}}$ denote the official training and test splits.

\paragraph{Query set and context pool.}
We construct a fixed query set $Q \subset \mathcal{D}_{\mathrm{te}}$ by stratified sampling such that class proportions match those of $\mathcal{D}_{\mathrm{te}}$, ensuring each class appears at least once. Concretely, $Q$ contains $10\%$ of $\mathcal{D}_{\mathrm{te}}$.
The context pool is defined as
\begin{equation*}
P \;=\; \mathcal{D}_{\mathrm{tr}} \cup \bigl(\mathcal{D}_{\mathrm{te}}\setminus Q\bigr),
\end{equation*}
so that $Q$ is disjoint from the context used for inference.

\paragraph{Context construction (class-balanced sampling).}
Given a context budget $N_{\mathrm{ctx}}$, we sample a context set $S_{N_{\mathrm{ctx}}}\subset P$ in a class-balanced manner.
Let $\mathcal{Y}=\{1,\ldots,C\}$ and $P_c=\{(x,y)\in P: y=c\}$.
We allocate $n_c=\left\lfloor N_{\mathrm{ctx}}/C \right\rfloor$ examples per class and distribute the remainder
$r=N_{\mathrm{ctx}}-Cn_c$ by adding one extra example to $r$ randomly chosen classes. If $|P_c|<n_c$, we use all available samples in $P_c$ and reallocate the remaining budget to other classes while preserving balance as much as possible.

\paragraph{Budgets and reporting.}
We sweep $N_{\mathrm{ctx}} \in \{N_0,5N_0,10N_0,15N_0,20N_0\}$, where $N_0=10C$ (roughly ten labeled examples per class at the smallest budget).
For each budget, we repeat context sampling five times with different random seeds and report the average accuracy on the fixed query set $Q$.

\subsection{Implementation and Hardware Specifications}
\label{app:imple}

\paragraph{Software Environment and Reproducibility.} 
To facilitate reproducibility, the complete source code is publicly available. Our method is implemented using \texttt{PyTorch} version 2.9.0 and \texttt{Python} 3.10. 
For the comparative time series foundation models, Mantis and MOMENT, we utilized their respective official open-source implementations and corresponding dependency configurations. 
The implementation of classifiers (e.g., SVM, RF, KNN) and evaluation metrics relies on \texttt{scikit-learn} version 1.7.0. 
Furthermore, the codebase features a modular design to facilitate the extension of new encoders or in-context adapters.

\paragraph{Hardware Infrastructure.} 
We utilized distinct hardware environments tailored to the computational demands of different experimental stages. 
The  pretraining phase was executed on a high-performance computing cluster equipped with 8 $\times$ AMD Instinct MI200 GPUs, each providing 64GB memory. 
In contrast, all subsequent experiments, including both the inference of our proposed method and the evaluation of all comparative baselines, were conducted on a computing node featuring 2 $\times$ NVIDIA GeForce RTX 3090 GPUs, each with 24GB of memory.

\paragraph{Architecture and hyperparameters.}
TIC-FM consists of a time series encoder $F_{\psi}$, a projection adapter $g_{\phi}$, and an in-context classifier $G_{\theta}$ with a latent-memory module. The encoder $F_{\psi}$ takes a univariate input sequence of length 512 and produces a 512-dimensional representation. It first applies a token generator built from two one-dimensional convolution layers with a kernel size of 17, and then aggregates the resulting tokens using a ViT backbone with 6 Transformer layers and dropout set to 0.1. The projection adapter is a 2-layer MLP with LayerNorm, hidden size $1024$, GELU, and dropout $0.118$ that preserves the feature dimension ($512\!\rightarrow\!1024\!\rightarrow\!512$).
The in-context classifier uses a $12$-block Transformer (4 heads, pre-norm) with a Perceiver-style latent memory of $32$ latents and $2$ write / $2$ read cross-attention blocks; labels are embedded via a one-hot linear map and decoded by an MLP ($512\!\rightarrow\!1024\!\rightarrow\!C_{\max}$) with $C_{\max}=10$. Detailed hyperparameters are provided in \cref{tab:arch_hparams}.


\begin{table}[h]
  \caption{\textbf{Architecture and key hyperparameters of TIC-FM.}}
  \label{tab:arch_hparams}
  \begin{center}
    \begin{small}
        \setlength{\tabcolsep}{4pt}
        \renewcommand{\arraystretch}{0.95}
        \begin{tabular*}{\columnwidth}{@{\extracolsep{\fill}}llr}
          \toprule
          \textbf{Module} & \textbf{Hyperparameter} & \textbf{Value} \\
          \midrule
          Encoder $F_{\psi}$ & Input length ($L$) & 512 \\
          & Feature dim ($d$) & 512 \\
          & Token-generator conv layers & 2 \\
          & Conv kernel size & 17 \\
          & ViT Transformer layers & 6 \\
          & Encoder dropout & 0.1 \\
          \midrule
          Adapter $g_{\phi}$ & MLP dims & $512 \!\rightarrow\! 1024 \!\rightarrow\! 512$ \\
          & Nonlinearity & GELU \\
          & LayerNorm & yes \\
          & Dropout & 0.118 \\
          \midrule
          ICL classifier $G_{\theta}$ & Transformer blocks & 12 \\
          & Attention heads & 4 \\
          & Norm-first (pre-norm) & yes \\
          \midrule
          Latent memory & \# latents & 32 \\
          & Write layers / Read layers & 2 / 2 \\
          \midrule
          Label space & $C_{\max}$ & 10 \\
          \bottomrule
        \end{tabular*}
    \end{small}
  \end{center}
\end{table}

\subsection{Additional Results and Statistical Analysis}
\label{app:additional_results}

\subsubsection{Per-dataset Results on the Full UCR Archive}

\label{app:mainsReluts}

Table~\ref{tab:ucr_full_results} reports per-dataset classification accuracy on all  UCR datasets for each backbone--classifier combination and TIC-FM. We include these results to complement the aggregate metrics in the main paper (average accuracy and mean rank) and to enable a fine-grained inspection of where improvements arise. In particular, per-dataset reporting helps assess whether gains are broadly distributed across datasets or concentrated in a small subset, and it makes explicit the sensitivity of freeze backbone and classifier evaluation to the choice of classifier.

Each entry is reported as the average accuracy computed over five independent runs with distinct random seeds. This protocol ensures that the reported performance is robust to the stochasticity inherent in classifier initialization and training.For deterministic classifiers under our protocol, the range can be zero. We emphasize that the range is included to expose the variability introduced by classifier training and data-dependent optimization, which is especially pronounced for non-convex heads such as MLP.

Two patterns are noteworthy. First, for a fixed backbone, the relative ordering of classifier heads can change substantially across datasets, indicating that downstream performance is often dominated by the classifier’s optimization behavior rather than solely by the frozen representations. Second, TIC-FM attains competitive or best performance on a large fraction of datasets without training a task-specific classifier, supporting our claim that inference-time conditioning provides a more reliable evaluation protocol when comparing time series foundation models.

\begingroup
\scriptsize
\setlength{\tabcolsep}{1.2pt}
\renewcommand{\arraystretch}{1.08}
\setlength{\LTleft}{0pt}
\setlength{\LTright}{0pt}

\begin{longtable}{@{\extracolsep{\fill}}l*{11}{r}}
\caption{Per-dataset classification accuracy (average) on the 128 UCR datasets. The best results are in \textbf{bold}.}
\label{tab:ucr_full_results}\\
\toprule
\multirow{2}{*}{\textbf{Dataset}} &
\multicolumn{5}{c}{\textbf{MOMENT}} &
\multicolumn{5}{c}{\textbf{Mantis}} &
\textbf{Ours} \\
\cmidrule(lr){2-6}\cmidrule(lr){7-11}
& \textbf{RF} & \textbf{SVM} & \textbf{MLP} & \textbf{kNN} & \textbf{NC}
& \textbf{RF} & \textbf{SVM} & \textbf{MLP} & \textbf{kNN} & \textbf{NC}
& \textbf{TIC-FM} \\
\midrule
\endfirsthead

\caption[]{Per-dataset classification accuracy (average ) on the 128 UCR datasets (continued). The best results are in \textbf{bold}.}\\
\toprule
\multirow{2}{*}{\textbf{Dataset}} &
\multicolumn{5}{c}{\textbf{MOMENT}} &
\multicolumn{5}{c}{\textbf{Mantis}} &
\textbf{Ours} \\
\cmidrule(lr){2-6}\cmidrule(lr){7-11}
& \textbf{RF} & \textbf{SVM} & \textbf{MLP} & \textbf{kNN} & \textbf{NC}
& \textbf{RF} & \textbf{SVM} & \textbf{MLP} & \textbf{kNN} & \textbf{NC}
& \textbf{TIC-FM} \\
\midrule
\endhead

\midrule
\multicolumn{12}{r}{\scriptsize Continued on next page} \\
\endfoot

\bottomrule
\endlastfoot

ACSF1 & \textbf{0.8040} & 0.6800 & 0.2560 & 0.7000 & 0.5800 & 0.7820 & 0.4760 & 0.3260 & 0.6800 & 0.5600 & 0.6700 \\
Adiac & \textbf{0.7918} & 0.2916 & 0.0261 & 0.7545 & 0.7212 & 0.7253 & 0.7673 & 0.4179 & 0.6547 & 0.5857 & 0.6573 \\
AllGestureWiimoteX & 0.6683 & 0.7129 & 0.3286 & 0.7043 & 0.5457 & 0.6609 & 0.6614 & 0.4060 & \textbf{0.7157} & 0.4286 & 0.7000 \\
AllGestureWiimoteY & 0.7023 & \textbf{0.7443} & 0.3454 & 0.7429 & 0.5543 & 0.6483 & 0.7186 & 0.3929 & 0.7357 & 0.3686 & 0.7243 \\
AllGestureWiimoteZ & 0.5840 & 0.6157 & 0.2389 & 0.5871 & 0.4171 & 0.6649 & 0.6686 & 0.3863 & 0.6386 & 0.3957 & \textbf{0.6686} \\
ArrowHead & \textbf{0.8229} & 0.6971 & 0.5577 & 0.7943 & 0.4571 & 0.7166 & 0.8103 & 0.5703 & 0.7543 & 0.5714 & 0.7771 \\
BME & 0.9720 & 0.9800 & 0.6240 & 0.9467 & 0.7933 & 0.9347 & \textbf{0.9933} & 0.6253 & 0.9533 & 0.6400 & 0.9533 \\
Beef & 0.7200 & 0.6000 & 0.3867 & 0.5667 & 0.5000 & 0.6533 & \textbf{0.7533} & 0.4733 & 0.6333 & 0.4333 & 0.7333 \\
BeetleFly & 0.9300 & \textbf{0.9500} & 0.8100 & \textbf{0.9500} & \textbf{0.9500} & 0.8300 & 0.8500 & 0.6900 & 0.8000 & 0.8000 & 0.9000 \\
BirdChicken & 0.8700 & 0.9000 & 0.7300 & 0.8500 & 0.8000 & 0.9900 & 0.9500 & 0.8800 & 0.9000 & \textbf{1.0000} & 0.7500 \\
CBF & 0.9240 & 0.9767 & 0.3313 & 0.9189 & 0.9167 & 0.9889 & \textbf{1.0000} & 0.8631 & \textbf{1.0000} & 0.9733 & 0.9989 \\
Car & 0.7600 & 0.7500 & 0.2133 & 0.8333 & 0.6333 & 0.7867 & \textbf{0.8700} & 0.5133 & 0.8333 & 0.6833 & 0.8333 \\
Chinatown & 0.9429 & 0.9650 & 0.6700 & 0.9329 & 0.9067 & 0.8426 & 0.9096 & 0.6974 & 0.7901 & 0.8630 & \textbf{0.9738} \\
ChlorineConcentration & 0.6822 & 0.5716 & 0.5326 & 0.6414 & 0.3104 & 0.6765 & \textbf{0.6888} & 0.5380 & 0.6242 & 0.3250 & 0.5656 \\
CinCECGTorso & 0.6662 & \textbf{0.7565} & 0.2790 & 0.7014 & 0.5529 & 0.6584 & 0.7539 & 0.4393 & 0.7246 & 0.5768 & 0.6964 \\
Coffee & 0.9429 & 0.8929 & 0.6500 & 0.9643 & 0.8929 & 0.9571 & \textbf{1.0000} & 0.7786 & 0.9286 & 0.9286 & \textbf{1.0000} \\
Computers & 0.7016 & 0.7280 & 0.6088 & 0.6800 & 0.5360 & 0.7288 & 0.7032 & 0.6560 & 0.6520 & 0.7240 & \textbf{0.7480} \\
CricketX & 0.6846 & 0.7154 & 0.1846 & 0.7231 & 0.5564 & 0.7328 & 0.6923 & 0.5667 & 0.7769 & 0.6462 & \textbf{0.8103} \\
CricketY & 0.6595 & 0.7308 & 0.1974 & 0.7154 & 0.5154 & 0.7374 & 0.7282 & 0.5544 & \textbf{0.8051} & 0.6308 & 0.7897 \\
CricketZ & 0.6846 & 0.7179 & 0.1087 & 0.6692 & 0.6026 & 0.7733 & 0.6795 & 0.5656 & \textbf{0.8179} & 0.6615 & 0.8077 \\
Crop & 0.6798 & \textbf{0.6994} & 0.4318 & 0.6662 & 0.4567 & 0.6689 & 0.6940 & 0.6395 & 0.6403 & 0.5346 & 0.6514 \\
DiatomSizeReduction & 0.8451 & 0.7092 & 0.3007 & \textbf{0.9542} & 0.8235 & 0.8575 & 0.8922 & 0.7745 & 0.9248 & 0.8954 & \textbf{0.9542} \\
DistalPhalanxOutlineAgeGroup & 0.7281 & 0.7554 & 0.4676 & 0.7050 & 0.7122 & \textbf{0.7885} & 0.7050 & 0.7439 & 0.7698 & 0.6978 & 0.7410 \\
DistalPhalanxOutlineCorrect & \textbf{0.8167} & 0.7862 & 0.5833 & 0.7355 & 0.4203 & 0.7543 & 0.7210 & 0.7659 & 0.7391 & 0.6775 & 0.7717 \\
DistalPhalanxTW & 0.6691 & 0.6691 & 0.3022 & 0.5827 & 0.6187 & 0.6820 & 0.6475 & 0.6345 & 0.6043 & 0.5396 & \textbf{0.6835} \\
DodgerLoopDay & 0.2975 & 0.4000 & 0.1825 & 0.2875 & 0.2500 & 0.4975 & 0.4975 & 0.3475 & 0.3875 & \textbf{0.5125} & 0.4750 \\
DodgerLoopGame & 0.7319 & \textbf{0.8406} & 0.6000 & 0.6812 & 0.5797 & 0.7246 & 0.7275 & 0.5696 & 0.7029 & 0.6087 & 0.5942 \\
DodgerLoopWeekend & 0.9043 & 0.9565 & 0.8203 & 0.8623 & 0.8696 & 0.9536 & 0.9710 & 0.9261 & 0.9493 & \textbf{0.9783} & 0.9420 \\
ECG200 & 0.8240 & \textbf{0.8700} & 0.6400 & 0.8100 & 0.7700 & 0.8220 & 0.8600 & 0.7520 & 0.8200 & 0.7700 & 0.8200 \\
ECG5000 & 0.9372 & \textbf{0.9447} & 0.8300 & 0.9222 & 0.8571 & 0.9213 & 0.9118 & 0.8929 & 0.9204 & 0.8051 & 0.9376 \\
ECGFiveDays & 0.7233 & 0.9721 & 0.4971 & 0.8908 & 0.5923 & 0.8997 & \textbf{0.9823} & 0.6246 & 0.8316 & 0.7085 & 0.9617 \\
EOGHorizontalSignal & 0.5613 & 0.5552 & 0.2989 & 0.5331 & 0.4696 & \textbf{0.5917} & 0.4724 & 0.4729 & 0.5773 & 0.5110 & 0.5801 \\
EOGVerticalSignal & 0.4823 & \textbf{0.5166} & 0.3155 & 0.4917 & 0.3978 & 0.4575 & 0.4558 & 0.3785 & 0.4779 & 0.3674 & 0.4834 \\
Earthquakes & 0.7439 & 0.7266 & 0.7482 & 0.6547 & 0.5468 & \textbf{0.7482} & 0.7050 & 0.7424 & 0.6906 & 0.4964 & 0.7482 \\
ElectricDevices & 0.7172 & \textbf{0.7513} & 0.5291 & 0.6832 & 0.5283 & 0.7228 & 0.7034 & 0.7012 & 0.6995 & 0.5818 & 0.7199 \\
EthanolLevel & 0.4356 & 0.3680 & 0.2848 & 0.3400 & 0.3100 & 0.2980 & \textbf{0.4940} & 0.2688 & 0.2560 & 0.2740 & 0.3100 \\
FaceAll & 0.7304 & \textbf{0.8077} & 0.2244 & 0.7349 & 0.4195 & 0.7807 & 0.8024 & 0.7564 & 0.7811 & 0.7775 & 0.7444 \\
FaceFour & 0.6886 & 0.7955 & 0.3023 & 0.7727 & 0.7955 & 0.9455 & \textbf{0.9659} & 0.7364 & 0.9545 & 0.9545 & 0.9205 \\
FacesUCR & 0.7355 & 0.8546 & 0.1374 & 0.7878 & 0.5541 & 0.8245 & 0.8840 & 0.4589 & \textbf{0.8932} & 0.8420 & 0.8893 \\
FiftyWords & 0.6422 & \textbf{0.7736} & 0.1266 & 0.6527 & 0.6000 & 0.6295 & 0.7516 & 0.4945 & 0.7275 & 0.6901 & 0.6923 \\
Fish & 0.8640 & 0.8629 & 0.1371 & 0.8686 & 0.8057 & 0.9383 & \textbf{0.9451} & 0.6274 & 0.9314 & 0.9143 & 0.8743 \\
FordA & 0.9262 & \textbf{0.9417} & 0.8577 & 0.8917 & 0.6841 & 0.8565 & 0.9106 & 0.8602 & 0.7939 & 0.7833 & 0.8962 \\
FordB & \textbf{0.8064} & 0.7938 & 0.6469 & 0.7481 & 0.5716 & 0.7341 & 0.7667 & 0.7420 & 0.6988 & 0.6407 & 0.7568 \\
FreezerRegularTrain & 0.8736 & 0.9488 & 0.7036 & 0.8449 & 0.7751 & 0.9354 & \textbf{0.9888} & 0.7874 & 0.9225 & 0.7572 & 0.9811 \\
FreezerSmallTrain & 0.7876 & 0.8077 & 0.6589 & 0.7688 & 0.7491 & 0.8022 & \textbf{0.8743} & 0.7535 & 0.7702 & 0.7453 & 0.8568 \\
Fungi & 0.9946 & \textbf{1.0000} & 0.3118 & \textbf{1.0000} & \textbf{1.0000} & 0.8022 & 0.8935 & 0.3817 & 0.8280 & 0.8280 & 0.7849 \\
GestureMidAirD1 & 0.6200 & 0.6538 & 0.2600 & 0.5846 & 0.5154 & 0.6462 & \textbf{0.7015} & 0.2631 & 0.6231 & 0.5538 & 0.6846 \\
GestureMidAirD2 & 0.5785 & 0.5231 & 0.2385 & 0.5308 & 0.4615 & 0.6138 & 0.5538 & 0.2585 & 0.5615 & 0.5615 & \textbf{0.6154} \\
GestureMidAirD3 & 0.3400 & 0.3538 & 0.1492 & \textbf{0.3692} & 0.2615 & 0.3400 & 0.3477 & 0.1538 & 0.3231 & 0.2615 & \textbf{0.3692} \\
GesturePebbleZ1 & 0.8779 & 0.9128 & 0.1930 & 0.8081 & 0.8140 & \textbf{0.9279} & 0.9186 & 0.7023 & 0.9070 & 0.8488 & 0.9070 \\
GesturePebbleZ2 & 0.8823 & 0.9114 & 0.3747 & 0.7405 & 0.8608 & 0.9241 & 0.8734 & 0.7266 & 0.8544 & \textbf{0.9494} & 0.8228 \\
GunPoint & 0.9840 & \textbf{0.9933} & 0.5440 & 0.9600 & 0.7333 & 0.9693 & 0.9867 & 0.8733 & 0.9800 & 0.9133 & \textbf{0.9933} \\
GunPointAgeSpan & 0.9766 & 0.9842 & 0.6424 & 0.9652 & 0.6361 & 0.9911 & \textbf{0.9949} & 0.8342 & 0.9905 & 0.9652 & 0.9842 \\
GunPointMaleVersusFemale & 0.9804 & 0.9842 & 0.5911 & 0.9810 & 0.9367 & \textbf{0.9968} & 0.9937 & 0.7987 & 0.9873 & 0.8703 & 0.9968 \\
GunPointOldVersusYoung & 0.9721 & 0.9492 & 0.5898 & 0.9683 & 0.5524 & 0.9968 & 0.9968 & 0.8267 & \textbf{1.0000} & 0.8603 & \textbf{1.0000} \\
Ham & 0.6724 & \textbf{0.7048} & 0.5657 & 0.5429 & 0.6476 & 0.6743 & 0.6590 & 0.5295 & 0.5524 & 0.6667 & 0.6571 \\
HandOutlines & 0.9097 & 0.9216 & 0.6405 & 0.8568 & 0.6784 & 0.9249 & \textbf{0.9324} & 0.8951 & 0.8757 & 0.6595 & 0.8811 \\
Haptics & 0.4929 & \textbf{0.4968} & 0.2136 & 0.4481 & 0.4156 & 0.4721 & 0.4481 & 0.3156 & 0.4416 & 0.4610 & 0.4481 \\
Herring & 0.6156 & 0.5938 & 0.5938 & 0.6094 & 0.5469 & 0.6375 & \textbf{0.7031} & 0.6250 & 0.5000 & 0.6719 & 0.6250 \\
HouseTwenty & 0.9445 & \textbf{0.9580} & 0.9143 & 0.9076 & 0.8908 & 0.9445 & 0.9496 & 0.8891 & 0.9412 & 0.9328 & 0.9496 \\
InlineSkate & 0.3593 & 0.3236 & 0.1622 & 0.4091 & 0.2345 & 0.3535 & 0.3298 & 0.2058 & 0.3782 & 0.2655 & \textbf{0.4273} \\
InsectEPGRegularTrain & 0.9510 & \textbf{1.0000} & 0.4763 & 0.9719 & 0.8916 & \textbf{1.0000} & \textbf{1.0000} & 0.9655 & \textbf{1.0000} & \textbf{1.0000} & \textbf{1.0000} \\
InsectEPGSmallTrain & 0.9309 & 0.9558 & 0.4763 & 0.8876 & 0.8594 & \textbf{1.0000} & \textbf{1.0000} & 0.9727 & \textbf{1.0000} & \textbf{1.0000} & 0.9799 \\
InsectWingbeatSound & 0.5391 & \textbf{0.5980} & 0.2957 & 0.4525 & 0.4152 & 0.5127 & 0.4651 & 0.3080 & 0.4369 & 0.4636 & 0.5162 \\
ItalyPowerDemand & 0.9320 & \textbf{0.9504} & 0.5621 & 0.9281 & 0.6453 & 0.9044 & 0.8939 & 0.7100 & 0.9213 & 0.8639 & 0.9174 \\
LargeKitchenAppliances & \textbf{0.8533} & 0.8480 & 0.7547 & 0.8320 & 0.7733 & 0.7904 & 0.8320 & 0.7045 & 0.7467 & 0.6533 & 0.7920 \\
Lightning2 & 0.7148 & 0.7541 & 0.5410 & 0.7705 & 0.5738 & 0.8033 & 0.7541 & 0.6131 & 0.8525 & 0.6885 & \textbf{0.8525} \\
Lightning7 & 0.6493 & 0.6849 & 0.2603 & 0.6438 & 0.7123 & 0.7534 & \textbf{0.7753} & 0.5397 & 0.6712 & 0.7260 & 0.7534 \\
Mallat & 0.8791 & 0.8768 & 0.1250 & 0.8857 & 0.9313 & 0.8829 & \textbf{0.9408} & 0.5167 & 0.9168 & 0.9356 & 0.9326 \\
Meat & 0.8867 & 0.8333 & 0.4000 & 0.8667 & 0.8333 & \textbf{0.9333} & 0.7633 & 0.4700 & 0.8500 & 0.9000 & 0.9000 \\
MedicalImages & 0.7361 & 0.7618 & 0.5145 & 0.7118 & 0.3039 & 0.6966 & 0.7053 & 0.5668 & 0.7053 & 0.4961 & \textbf{0.7829} \\
MelbournePedestrian & 0.8379 & 0.8421 & 0.4863 & 0.8233 & 0.5469 & 0.8999 & 0.9192 & 0.8376 & 0.8680 & 0.7589 & \textbf{0.9582} \\
MiddlePhalanxOutlineAgeGroup & 0.5636 & 0.6299 & 0.1883 & 0.5390 & 0.5779 & 0.5870 & 0.5195 & 0.5494 & 0.5000 & 0.5195 & \textbf{0.6364} \\
MiddlePhalanxOutlineCorrect & \textbf{0.8323} & 0.6529 & 0.5704 & 0.7388 & 0.6289 & 0.8055 & 0.8110 & 0.7024 & 0.7251 & 0.5911 & 0.8144 \\
MiddlePhalanxTW & 0.5623 & \textbf{0.5974} & 0.2727 & 0.4740 & 0.4675 & 0.5260 & 0.4481 & 0.5065 & 0.4870 & 0.3506 & 0.5714 \\
MixedShapesRegularTrain & 0.9235 & 0.9460 & 0.7722 & 0.9200 & 0.8186 & 0.9391 & 0.9431 & 0.8695 & \textbf{0.9530} & 0.8899 & 0.9443 \\
MixedShapesSmallTrain & 0.8646 & 0.8874 & 0.5742 & 0.8412 & 0.7996 & 0.8884 & 0.8957 & 0.7485 & 0.9105 & 0.8763 & \textbf{0.9122} \\
MoteStrain & 0.8909 & 0.8986 & 0.7455 & 0.8642 & 0.8203 & 0.9059 & 0.8818 & 0.8184 & 0.8834 & \textbf{0.9265} & 0.9185 \\
NonInvasiveFetalECGThorax1 & 0.8821 & \textbf{0.9033} & 0.0729 & 0.8539 & 0.8326 & 0.6159 & 0.8056 & 0.5421 & 0.4926 & 0.4539 & 0.8504 \\
NonInvasiveFetalECGThorax2 & 0.9122 & \textbf{0.9191} & 0.0703 & 0.8779 & 0.8656 & 0.6778 & 0.8412 & 0.5805 & 0.5878 & 0.5033 & 0.8718 \\
OSULeaf & 0.8562 & \textbf{0.8926} & 0.1818 & 0.8554 & 0.8099 & 0.8636 & 0.8843 & 0.6256 & 0.8802 & 0.8802 & 0.8430 \\
OliveOil & 0.8667 & 0.4000 & 0.4000 & 0.8333 & 0.7667 & \textbf{0.9133} & 0.5133 & 0.3933 & 0.8667 & 0.8667 & 0.6333 \\
PLAID & 0.7024 & 0.7523 & 0.2484 & 0.6909 & 0.2793 & 0.8086 & 0.8436 & 0.4737 & \textbf{0.8566} & 0.3110 & 0.7635 \\
PhalangesOutlinesCorrect & \textbf{0.8368} & 0.7016 & 0.6131 & 0.7716 & 0.6131 & 0.7699 & 0.7949 & 0.7378 & 0.7179 & 0.6131 & 0.7506 \\
Phoneme & 0.2782 & 0.2764 & 0.1113 & 0.2315 & 0.1493 & \textbf{0.3270} & 0.2608 & 0.1905 & 0.2716 & 0.2384 & 0.3149 \\
PickupGestureWiimoteZ & 0.6880 & 0.7200 & 0.3280 & 0.6600 & 0.6800 & \textbf{0.7920} & 0.7360 & 0.4520 & 0.7200 & 0.6800 & 0.7200 \\
PigAirwayPressure & 0.1135 & 0.0817 & 0.0481 & 0.1106 & 0.0721 & 0.4606 & 0.4548 & 0.1279 & 0.4423 & 0.4808 & \textbf{0.7067} \\
PigArtPressure & 0.7923 & 0.8606 & 0.2337 & 0.8510 & 0.7644 & 0.8885 & 0.7558 & 0.2077 & 0.9038 & 0.8365 & \textbf{0.9375} \\
PigCVP & 0.7317 & 0.7981 & 0.2654 & 0.7788 & 0.7500 & 0.7644 & 0.7904 & 0.2933 & 0.8173 & 0.7692 & \textbf{0.8750} \\
Plane & \textbf{1.0000} & \textbf{1.0000} & 0.0952 & \textbf{1.0000} & \textbf{1.0000} & \textbf{1.0000} & \textbf{1.0000} & 0.9162 & \textbf{1.0000} & \textbf{1.0000} & \textbf{1.0000} \\
PowerCons & 0.8922 & 0.9000 & 0.7067 & 0.8389 & 0.7278 & 0.9144 & 0.9189 & 0.8000 & 0.9333 & 0.8500 & \textbf{0.9611} \\
ProximalPhalanxOutlineAgeGroup & 0.8400 & 0.8488 & 0.4878 & 0.8098 & 0.8488 & \textbf{0.8576} & 0.8390 & 0.8517 & 0.7415 & 0.8146 & 0.8341 \\
ProximalPhalanxOutlineCorrect & \textbf{0.8763} & 0.7595 & 0.6838 & 0.8351 & 0.6392 & 0.8131 & 0.8419 & 0.7766 & 0.7663 & 0.6460 & 0.7801 \\
ProximalPhalanxTW & 0.8078 & \textbf{0.8098} & 0.3512 & 0.7415 & 0.7366 & 0.7659 & 0.6829 & 0.7532 & 0.7073 & 0.6098 & 0.7902 \\
RefrigerationDevices & 0.5392 & 0.4987 & 0.5163 & 0.4720 & 0.5013 & 0.5019 & \textbf{0.5413} & 0.4923 & 0.4747 & 0.5227 & 0.5387 \\
Rock & 0.6640 & 0.7400 & 0.5400 & 0.7600 & 0.5200 & 0.7640 & 0.8200 & 0.5960 & \textbf{0.8400} & 0.7200 & 0.6400 \\
ScreenType & 0.5056 & \textbf{0.5360} & 0.4123 & 0.4827 & 0.4240 & 0.4464 & 0.4373 & 0.4133 & 0.3787 & 0.4213 & 0.5173 \\
SemgHandGenderCh2 & 0.7443 & 0.7617 & 0.6317 & 0.6767 & 0.6567 & 0.8937 & 0.8700 & 0.7253 & 0.8850 & 0.7000 & \textbf{0.9183} \\
SemgHandMovementCh2 & 0.3876 & 0.4200 & 0.3084 & 0.3511 & 0.3289 & 0.7209 & 0.6644 & 0.4276 & 0.7222 & 0.4378 & \textbf{0.7489} \\
SemgHandSubjectCh2 & 0.5956 & 0.6556 & 0.4853 & 0.5311 & 0.5178 & 0.8031 & \textbf{0.8578} & 0.6018 & 0.8156 & 0.5156 & 0.8400 \\
ShakeGestureWiimoteZ & 0.9040 & 0.9000 & 0.6160 & 0.8000 & 0.9200 & 0.8840 & 0.8840 & 0.6720 & 0.8800 & 0.8800 & \textbf{0.9400} \\
ShapeletSim & 0.9322 & \textbf{0.9722} & 0.8511 & 0.9333 & 0.9000 & 0.9200 & 0.9222 & 0.7433 & 0.9056 & 0.9056 & 0.7944 \\
ShapesAll & 0.8610 & \textbf{0.8733} & 0.1040 & 0.8600 & 0.7650 & 0.8140 & 0.8317 & 0.7310 & 0.8650 & 0.8050 & 0.8483 \\
SmallKitchenAppliances & 0.7819 & 0.7440 & 0.6395 & 0.6667 & 0.6587 & 0.8112 & 0.7920 & 0.7989 & 0.7467 & 0.7920 & \textbf{0.8213} \\
SmoothSubspace & 0.9267 & \textbf{0.9667} & 0.6840 & 0.9000 & 0.8667 & 0.9080 & 0.9333 & 0.7173 & 0.8467 & 0.9267 & 0.9333 \\
SonyAIBORobotSurface1 & 0.8735 & 0.8968 & 0.4293 & 0.8985 & \textbf{0.9151} & 0.7704 & 0.8270 & 0.6296 & 0.8087 & 0.9018 & 0.7188 \\
SonyAIBORobotSurface2 & 0.9123 & \textbf{0.9570} & 0.6170 & 0.9129 & 0.8982 & 0.8306 & 0.9081 & 0.7343 & 0.8898 & 0.8416 & 0.8804 \\
StarLightCurves & 0.9763 & 0.9733 & 0.7364 & 0.9649 & 0.8738 & 0.9761 & 0.9648 & 0.9731 & 0.9703 & 0.9526 & \textbf{0.9796} \\
Strawberry & \textbf{0.9654} & 0.9216 & 0.6432 & 0.9595 & 0.5568 & 0.9503 & 0.9486 & 0.8654 & 0.9378 & 0.7892 & 0.9405 \\
SwedishLeaf & 0.9158 & 0.9376 & 0.0531 & 0.9072 & 0.8240 & 0.9274 & \textbf{0.9424} & 0.8125 & 0.9056 & 0.8880 & 0.9296 \\
Symbols & 0.9622 & 0.9678 & 0.1906 & 0.9608 & 0.9558 & 0.9574 & \textbf{0.9869} & 0.8953 & 0.9779 & 0.9628 & 0.9598 \\
SyntheticControl & 0.9340 & 0.9633 & 0.6500 & 0.9233 & 0.8633 & 0.9753 & 0.9867 & 0.9193 & 0.9733 & 0.9733 & \textbf{0.9933} \\
ToeSegmentation1 & 0.9351 & 0.9386 & 0.8316 & 0.9035 & 0.8772 & \textbf{0.9649} & 0.9605 & 0.8825 & 0.8991 & \textbf{0.9649} & 0.9035 \\
ToeSegmentation2 & 0.9077 & 0.9231 & 0.9138 & 0.9385 & 0.9077 & 0.9200 & \textbf{0.9462} & 0.8631 & \textbf{0.9462} & 0.9385 & 0.9077 \\
Trace & \textbf{1.0000} & \textbf{1.0000} & 0.1900 & \textbf{1.0000} & 0.9900 & \textbf{1.0000} & \textbf{1.0000} & 0.8640 & \textbf{1.0000} & \textbf{1.0000} & \textbf{1.0000} \\
TwoLeadECG & 0.9861 & 0.9956 & 0.4997 & 0.9789 & 0.9464 & 0.9961 & \textbf{0.9963} & 0.7716 & 0.9860 & 0.9877 & 0.9903 \\
TwoPatterns & 0.8866 & 0.9838 & 0.4228 & 0.8293 & 0.7540 & 0.8708 & 0.9670 & 0.8007 & 0.8350 & 0.7275 & \textbf{0.9852} \\
UMD & 0.9806 & 0.9792 & 0.6542 & 0.9861 & 0.8194 & 0.9694 & 0.9931 & 0.7472 & 0.9931 & 0.6667 & \textbf{0.9931} \\
UWaveGestureLibraryAll & 0.8257 & \textbf{0.9227} & 0.4342 & 0.8205 & 0.6309 & 0.8382 & 0.8814 & 0.8012 & 0.8504 & 0.7741 & 0.8889 \\
UWaveGestureLibraryX & 0.7518 & 0.7954 & 0.4577 & 0.7406 & 0.6667 & 0.7614 & 0.7647 & 0.7249 & 0.7507 & 0.7164 & \textbf{0.8132} \\
UWaveGestureLibraryY & 0.6901 & 0.7281 & 0.3266 & 0.6616 & 0.5103 & 0.6747 & 0.6647 & 0.6386 & 0.6907 & 0.6287 & \textbf{0.7426} \\
UWaveGestureLibraryZ & 0.7061 & 0.7398 & 0.3777 & 0.6904 & 0.5771 & 0.7225 & 0.7365 & 0.6853 & 0.7083 & 0.6672 & \textbf{0.7661} \\
Wafer & 0.9789 & 0.9974 & 0.8921 & 0.9825 & 0.8251 & 0.9903 & \textbf{0.9976} & 0.9581 & 0.9903 & 0.8214 & 0.9916 \\
Wine & 0.6667 & 0.5000 & 0.5000 & 0.5741 & 0.5185 & 0.7667 & 0.5000 & 0.5000 & 0.5926 & 0.7037 & \textbf{0.8704} \\
WordSynonyms & 0.5404 & 0.6160 & 0.2194 & 0.5940 & 0.2633 & 0.5448 & 0.6285 & 0.3273 & \textbf{0.6928} & 0.4734 & 0.6176 \\
Worms & \textbf{0.8442} & 0.7792 & 0.4286 & 0.7662 & 0.7922 & 0.6260 & 0.6753 & 0.5039 & 0.6104 & 0.7143 & 0.6753 \\
WormsTwoClass & \textbf{0.8364} & 0.8052 & 0.5714 & 0.8312 & 0.7532 & 0.7922 & 0.7818 & 0.6753 & 0.6883 & 0.6234 & 0.7662 \\
Yoga & 0.8057 & 0.7300 & 0.5357 & 0.8400 & 0.5573 & 0.8099 & 0.7850 & 0.6311 & 0.8267 & 0.6167 & \textbf{0.8527} \\

% ...

\end{longtable}

\endgroup


\subsubsection{Extended Analysis of Scalability with Supervision Budgets}
\label{app:trainFra}
In Figure~\ref{fig:train_fraction_scaling_app}, we present the training-fraction scaling curves for the baseline configurations that are omitted from the main text (Figure~\ref{fig:train_fraction_scaling}) for clarity.
As illustrated, {TIC-FM consistently achieves the highest accuracy across all label fractions}, demonstrating its robustness compared to both parametric (e.g., MLP) and non-parametric (e.g., KNN, NC) classifiers.
Detailed numerical results for all methods under varying supervision budgets are provided in Table~\ref{tab:train_fraction_scaling}.

\begin{figure}[t]
  \centering
   \includegraphics[width=0.80\linewidth]{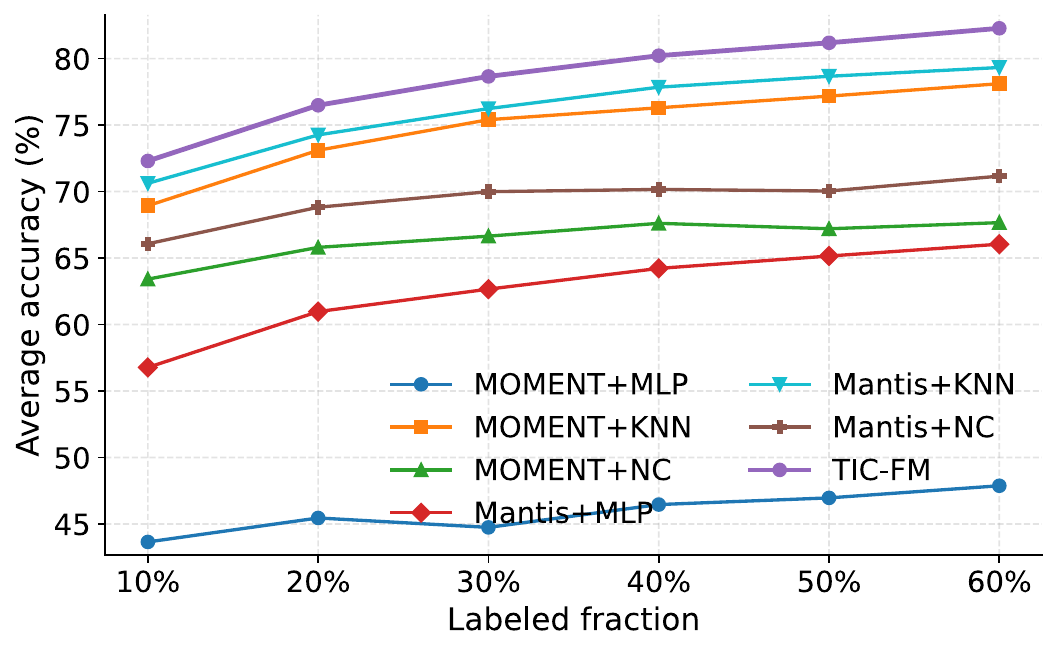} 
  \caption{\textbf{Extended scalability analysis with varying labeled data fractions.} This figure complements Figure~\ref{fig:train_fraction_scaling} by illustrating the performance of TIC-FM against the full set of baseline configurations, including parametric (e.g., MLP) and non-parametric (e.g., KNN, NC) classifiers omitted from the main text. \textbf{Observations:} TIC-FM consistently outperforms all baseline variants across all supervision budgets.}
  \label{fig:train_fraction_scaling_app}
\end{figure}

\begin{table}[h]
\centering
\caption{Average accuracy under different labeled fractions on the UCR test split.The best results are in \textbf{bold}.}
\label{tab:train_fraction_scaling}
\begin{tabular}{lcccccc}
\toprule
\textbf{Method} & \textbf{0.1} & \textbf{0.2} & \textbf{0.3} & \textbf{0.4} & \textbf{0.5} & \textbf{0.6} \\
\midrule
MOMENT+RF  & 0.6919 & 0.7460 & 0.7601 & 0.7839 & 0.7878 & 0.8030 \\
MOMENT+SVM & 0.7013 & 0.7530 & 0.7696 & 0.7869 & 0.7836 & 0.7988 \\
MOMENT+MLP & 0.4364 & 0.4545 & 0.4474 & 0.4646 & 0.4696 & 0.4787 \\
MOMENT+KNN & 0.6893 & 0.7310 & 0.7540 & 0.7630 & 0.7718 & 0.7811 \\
MOMENT+NC  & 0.6341 & 0.6580 & 0.6665 & 0.6761 & 0.6720 & 0.6766 \\
Mantis+RF  & 0.6957 & 0.7492 & 0.7687 & 0.7871 & 0.7923 & 0.8092 \\
Mantis+SVM & 0.7117 & 0.7482 & 0.7715 & 0.7845 & 0.7945 & 0.8079 \\
Mantis+MLP & 0.5677 & 0.6097 & 0.6266 & 0.6422 & 0.6515 & 0.6604 \\
Mantis+KNN & 0.7062 & 0.7426 & 0.7624 & 0.7785 & 0.7867 & 0.7933 \\
Mantis+NC  & 0.6607 & 0.6882 & 0.6999 & 0.7016 & 0.7004 & 0.7116 \\
\textbf{TIC-FM} & \textbf{0.7230} & \textbf{0.7649} & \textbf{0.7866} & \textbf{0.8022} & \textbf{0.8119} & \textbf{0.8228} \\
\bottomrule
\end{tabular}
\end{table}

\subsubsection{Extended Analysis of Context Window Size}
\label{app:contextWindow}

Table~\ref{tab:context_length_scaling_values} reports the numerical values corresponding to Figure~\ref{fig:context_length_scaling}, where the context budget is parameterized by the multiplier $m=N_{\mathrm{ctx}}/N_0$ (with $N_0=10C$ in our setup). Overall, increasing the number of labeled context examples consistently improves query accuracy across all three datasets (Crop, ElectricDevices, and ECG5000). This finding supports the central premise of TIC-FM, demonstrating that task adaptation is effectively achieved through inference-time conditioning on labeled context rather than by training a task-specific classifier. From an optimization perspective, enlarging the context budget yields a more reliable empirical estimate of the class-conditional structure, thereby enhancing the matching between context and query instances and reducing decision ambiguity.

Across datasets, the gains are most pronounced in the low-context regime. On ElectricDevices, accuracy increases sharply from 58.97\% at $m{=}1$ to 77.86\% at $m{=}5$, indicating that modest additional supervision can substantially enhance in-context reasoning when labeled support is scarce. Beyond $m{=}10$, improvements become incremental (81.24\% $\rightarrow$ 82.19\%), suggesting diminishing returns once the context set becomes sufficiently representative. A similar trend holds for Crop, which improves from 53.91\% at $m{=}1$ to 61.96\% at $m{=}5$, followed by gradual saturation (63.97\% at $m{=}10$ and 65.02\% at $m{=}20$). ECG5000 exhibits the same pattern with a stronger overall baseline, rising from 69.69\% at $m{=}1$ to 86.53\% at $m{=}5$, and continuing to improve more moderately as $m$ increases (90.09\% at $m{=}10$ and 92.27\% at $m{=}20$). This behavior aligns with the analysis in Section~\ref{sec:theory}: larger context budgets provide more informative conditioning signals that induce optimization-like refinement of activations, with decreasing marginal benefit once contextual evidence becomes strong.

\begin{table}[H]
\centering
\caption{Context length scaling results (accuracy, \%) on Crop and ElectricDevices. The context multiplier is $m=N_{\mathrm{ctx}}/N_0$.}
\label{tab:context_length_scaling_values}
\begin{tabular}{c ccc}
\toprule
$m$ & Crop (C=24, $N_0$=240) & ElectricDevices (C=7, $N_0$=70)  & ECG5000 (C=5, $N_0$=50) \\
\midrule
1  & 53.91 & 58.97 & 69.69\\
5  & 61.96 & 77.86 & 86.53 \\
10 & 63.97 & 81.24 & 90.09 \\
15 & 64.70 & 81.32 & 91.73 \\
20 & 65.02 & 82.19  & 92.27 \\
\bottomrule
\end{tabular}
\end{table}

\end{document}